\documentclass[twoside,11pt]{article}

\usepackage{blindtext}

\usepackage{jmlr2e}

\usepackage{amsmath,amsfonts,bm}

\def\eqref#1{equation~\ref{#1}}

\def\1{\bm{1}}

\def\vb{{\bm{b}}}

\def\vh{{\bm{h}}}

\def\vr{{\bm{r}}}
\def\vs{{\bm{s}}}

\def\vu{{\bm{u}}}
\def\vv{{\bm{v}}}
\def\vw{{\bm{w}}}
\def\vx{{\bm{x}}}

\def\vz{{\bm{z}}}

\DeclareMathAlphabet{\mathsfit}{\encodingdefault}{\sfdefault}{m}{sl}
\SetMathAlphabet{\mathsfit}{bold}{\encodingdefault}{\sfdefault}{bx}{n}
\newcommand{\tens}[1]{\bm{\mathsfit{#1}}}
\def\tA{{\tens{A}}}

\def\sC{{\mathbb{C}}}
\def\sD{{\mathbb{D}}}

\def\sI{{\mathbb{I}}}

\def\sL{{\mathbb{L}}}

\def\sR{{\mathbb{R}}}

\newcommand{\R}{\mathbb{R}}

\usepackage{booktabs}
\usepackage{hyperref}
\usepackage{url}
\usepackage{tikz}
\usepackage{cancel}
\usepackage{algorithm}
\usepackage{algpseudocode}
\usepackage{subcaption}
\usepackage{multirow}
\usepackage{xspace}
\usepackage{xcolor}
\usepackage{ifthen}

\newcommand{\dt}{\texttt{DTSemNet}\xspace} 

\usepackage{lastpage}
\jmlrheading{27}{2026}{1-\pageref{LastPage}}{8/25; Revised
4/26}{5/26}{25-2047}{Subrat Prasad Panda, Blaise Genest and Arvind Easwaran}

\ShortHeadings{Approximation-Free Differentiable Oblique Decision Trees}{Panda, Genest and Easwaran}
\firstpageno{1}

\begin{document}

\title{Approximation-Free Differentiable Oblique Decision Trees}

\author{%
  \name Subrat Prasad Panda \email subratpr001@e.ntu.edu.sg \\
  \addr College of Computing and Data Science\\
        Nanyang Technological University\\
        Singapore
  \AND
  \name Blaise Genest \email blaise.genest@cnrsatcreate.sg \\
  \addr CNRS@CREATE\\
        Singapore\\
        IPAL, CNRS\\
        France
  \AND
  \name Arvind Easwaran \email arvinde@ntu.edu.sg \\
  \addr College of Computing and Data Science\\
        Nanyang Technological University\\
        Singapore
}

\editor{Honglak Lee}

\maketitle

\begin{abstract}
Decision Trees (DTs) are widely used in safety-critical domains such as medical diagnosis, valued for their interpretability and effectiveness on tabular data. However, training accurate oblique DTs is challenging due to complex optimization landscapes and overfitting risks, particularly in regression. Recent advances have introduced differentiable formulations that enable gradient-based training and joint optimization of decision boundaries and leaf regressors. Yet, existing approaches typically rely on approximations, either through probabilistic softening of boundaries (soft DTs) or quantized gradients such as the Straight-Through Estimator (STE).
To overcome these limitations, we propose \dt, a novel, \underline{sem}antically equivalent, and invertible representation of hard oblique \underline{DT}s as neural \underline{net}works. \dt enables end-to-end training with standard gradient descent, eliminating the need for approximations in both classification and regression. While classification aligns naturally with this formulation, regression remains challenging due to the joint optimization of internal nodes and leaf regressors. To address this, we analyze the limitations of STE and introduce an annealed Top-$k$ method that provides accurate gradient signals without approximation.
Extensive experiments on classification and regression benchmarks show that \dt-trained oblique DTs outperform state-of-the-art differentiable DTs. Furthermore, we demonstrate that \dt can serve as programmatic DT policies in reinforcement learning environments, thereby broadening their applicability.
\end{abstract}


\section{Introduction}\label{sec:intro}

Decision Trees (DTs) are widely used across various domains, including high-stakes applications such as medical diagnostics \citep{chen2017machine}. Recent studies \citep{grinsztajn2022tree} highlight that DTs often outperform neural networks (NNs) on tabular data, owing to their strong inductive bias for modeling non-smooth functions. Despite their advantages, learning optimal DTs remains a challenging task. The search space grows combinatorially with the number of features and decision points, which makes learning an optimal DT NP-hard \citep{laurent1976constructing}.

Existing methods for learning DTs can be broadly grouped into non-gradient-based and gradient-based approaches. Non-gradient methods include greedy algorithms like CART \citep{breiman1984cart}, global optimization via MIP \citep{bennett1996optimal} or evolutionary algorithms \citep{ea2023}, and non-greedy methods like TAO \citep{carreira2018alternating, tao-regression}. These methods either yield suboptimal trees \citep{zantedeschi2021learning, dgt2022} or scale poorly with tree complexity \citep{bertsimas2017optimal, dgt2022}. In contrast, gradient-based methods make DTs differentiable, enabling end-to-end training via backpropagation \citep{hazimeh2020tree, yang2018deep, zantedeschi2021learning}. These have shown improved efficiency and accuracy for both classification and regression tasks \citep{dgt2022, icct2022, marton2024gradtree}. 

To use gradient descent at the training stage on DTs, most previous works adopt ``soft'' decisions (e.g., using the \texttt{Sigmoid} activation function) to make the tree differentiable \citep{lee2019oblique,biau2019neural,zantedeschi2021learning,hazimeh2020tree}. However, the resulting tree does not provide ``hard'' decisions but ``soft'', that is, probabilistic ones. Hard DTs can be derived from soft DTs through hardening, albeit at the cost of accuracy.
Some recent studies, such as Dense Gradient Trees (DGT) \citep{dgt2022}, Interpretable Continuous Control Trees (ICCT) \citep{icct2022}, and GradTree \citep{marton2024gradtree, marton2025mitigating}, have introduced an alternative approach for obtaining hard DTs by using an approximation during backpropagation to compute gradients with Straight-Through Estimators (STEs) \citep{bengio2013estimating,hubara2016binarized}. 
This approximation hampers DT training as the {\em evaluation function} computing the output during forward propagation and the {\em optimizing function} computing gradient during backward propagation are different.
This is especially true in large datasets or in Reinforcement Learning (RL) tasks, where errors may accumulate over many training steps. 

In this work, we focus on producing an approximation-free architecture, that is an architecture where the {\em evaluation} and the {\em optimizing} functions are the same.\footnote{This was not the case for the regression task in the conference version \citep{Panda2024}}
We propose a novel encoding of oblique DTs as NNs, namely the \dt architecture, which overcomes the aforementioned shortcomings. It uses \texttt{ReLU} activation functions and linear operations, making it differentiable and enabling the use of vanilla gradient descent to train the DT. The proposed encoding is semantically equivalent to a hard oblique DT, with each decision (weight) in the DT corresponding one-to-one to a trainable weight in the NN, while all other weights in the architecture are fixed (non-trainable).
\dt can be used in both classification and regression settings. In classification, it selects the class with the highest output using a standard \texttt{argmax} operation, and we show in Theorem \ref{th1} that, for any given input, the output of \dt is identical to that of the DT. In regression, each leaf of the DT is associated with a regressor, and \dt selects the appropriate leaf regressor, which is then used to compute the final output.

Training DTs for regression is more challenging than in the classification setting, as it requires jointly learning the parameters of DT (acting as the router) so that it selects the appropriate leaf, along with the parameters of the corresponding leaf regressor.
Following prior work, DGT and ICCT, \cite{Panda2024} used \dt with a single STE during training to select both the leaf and its corresponding output. However, STE introduces approximation due to biased gradient updates to the internal node parameters, caused by a {\em mismatch between the forward and backward propagation objectives}.
This mismatch gives rise to two key issues that degrade performance:
(i) internal node parameters are updated based on the outputs of all regressors at each step, which destabilizes the tree structure and reduces the accuracy of routing samples to the correct leaves; and
(ii) only a single regressor is updated per step, so if a sample is later routed to a different leaf, the corresponding regressor (being untrained for that sample) incurs a high loss. The tree therefore tends to retain the previously used leaf, since its loss is comparatively lower and reverting requires fewer updates than training a new regressor. This, in turn, results in repeated assignment of samples to the same leaf and underutilization of many leaves.
Together, these issues substantially impair the performance of DTs trained with STE.

To address the above limitations in regression, this work introduces a method inspired by the Top-$k$ selection mechanism widely used in Mixture-of-Experts (MoE) models. In this mechanism, a routing network selects a subset of $k$ experts, and their outputs are combined to produce the final result.
Indeed, DT regression can be viewed similarly, with the DT acting as a Top-$1$ router (in the form of a classifier) that selects a single expert (leaf), which is then used to compute the final output.
However, training a DT (represented by \dt) with Top-$1$ selection does not provide informative gradients.
In the MoE setting, it is common to use a fixed $k \geq 2$ (often $k = 4$ or $k = 5$), which has been shown to outperform STE-based approaches \citep{muqeethsoft}. Motivated by this, our method starts with $k \geq 2$ when training the router \dt and gradually anneals it to $k = 1$, yielding the hard DT required for \dt-regression.
This annealing strategy provides more stable and targeted updates by adjusting the parameters of the internal nodes of \dt based only on the outputs of the selected regressors.
Using the formal definition of Top-$k$, we show two key differences in its gradient updates compared to STE. First, in Top-$k$, the top $k$ regressors are updated, whereas STE updates only a single regressor. Second, the parameters of the DT are updated using only the top $k$ selected regressors, in contrast to STE, which updates based on all regressors.
Because Top-$k$ trains multiple regressors for each sample during training, it enables neighboring regressors to adapt when the router DT shifts a sample from one leaf to another. By gradually annealing $k$ from $k \geq 2$ to $k = 1$, our approach ensures that at each step $k \in \{K, \dots, 1\}$ there is {\em no approximation}, since the forward and backward passes share the same semantics.
In particular, for any fixed $k$, the same function is used to compute the output in forward propagation (evaluation) and the gradients in backward propagation (optimization), thereby maintaining an approximation-free training process throughout.
This preserves a consistent training objective for Top-$k$ and thereby improves DT regression performance, unlike STE, which suffers from a mismatch between forward and backward propagation.
From another perspective, the Top-$k$ annealing process can be regarded as a smart initialization of step $k$, using the parameters from a previous step $k' > k$.

We experimented with \dt on various benchmarks for both classification and regression tasks, comparing its performance against different methods for producing hard DTs: DGT \citep{dgt2022} for gradient descent-based approaches, TAO \citep{carreira2018alternating, tao-regression} for non-greedy algorithms, CRO-DT \citep{ea2023} for evolutionary algorithms, and CART as the standard greedy algorithm. The classification and regression tasks primarily involve tabular datasets, except for MNIST, a small image dataset. The number of features ranges from 4 to 780 (with MNIST as an outlier), and the number of classes ranges from 2 to 26. For classification, \dt achieves the best performance on every dataset, demonstrating the effectiveness of our proposed unapproximated methodology.
For regression, \dt outperforms all methods on all datasets except two, for which it produces DTs almost as accurate as the best ones ($<3\%$ and $<1\%$ from the best). Additionally, the training time for gradient-based DTs is significantly shorter than that of competing methods.

Finally, we demonstrate the deployment of \dt as a programmatic DT policy in RL environments with both discrete and continuous action spaces, achieved by simply replacing the NN. This aligns with the growing interest in Differentiable Programmatic RL to represent policies as programs \citep{verma2018programmatically, qiu2021programmatic}. Gradient-based DT training is particularly well-suited for RL, where training without gradients is challenging due to the absence of fixed datasets.
A common workaround is to train an NN and then distill it into a DT using greedy strategies such as CART, which performs well only in simple RL tasks with low-dimensional, discrete action spaces \citep{bastani2018verifiable}. In contrast, gradient-based DTs can directly learn policies for RL tasks and match the efficiency of NNs \citep{dgt2022, icct2022, marton2025mitigating}. \dt can replace NNs in standard RL pipelines (e.g., Stable-Baselines3 \citep{stable-baselines3} or CleanRL \citep{huang2022cleanrl}). On discrete-action benchmarks (up to 10 actions, 32 input dimensions), \dt matches NN performance and significantly outperforms other DT-based methods trained using gradient descent or imitation learning. For continuous actions with limited input dimensions, we modify the Soft Actor-Critic (SAC) \citep{sac} training loop to use Top-$k$ methods for DT training. \dt Top-$k$ outperforms competitors in continuous control, though architecture choice is less critical in this setting.

\smallskip
\noindent 
\textbf{Main contributions} are as follows:
\begin{itemize}
 \item We introduce the \dt architecture, which we prove to be semantically equivalent to DTs in Theorem 1.
 
 \item Learning \dt for classification with standard gradient descent corresponds exactly to learning a DT, without resorting to any approximations, unlike competing methods. Experimentally, this results in the most efficient DTs for classification tasks in every standard benchmark we experimented with.
 
 \item We develop an annealed Top-$k$ approach for using \dt in regression tasks that avoids approximation and overcomes the biased gradients and structural instability caused by STE approximation. This procedure provides a stable and accurate differentiable \dt, achieving the best performance on almost all benchmarks and coming within $<3\%$ of the best on the remaining benchmarks.

 \item  We explain how to use \dt as a programmatic DT policy in the RL setting, again producing experimentally the most efficient DT policies, for both continuous action spaces using \dt-regression; and by a large margin for discrete actions using \dt-classification.

 \end{itemize}

This paper is an extended version of our conference publication \citep{Panda2024}, offering substantial advances in the regression setting through a novel, approximation-free annealed Top-$k$ approach, introduced in Section~\ref{dt-regression}. In this new section, we formally define Top-$k$ and systematically compare its gradient updates with those of the STE-based method, thereby addressing the inherent limitations of STE-based approximation that constrained our earlier work \citep{Panda2024}. 
Experimental results, presented in Section~\ref{results_regression} with twice as many benchmarks as in \citep{Panda2024},
demonstrate the clear superiority of Top-$k$ for \dt-based regression across all benchmark datasets. 
While \dt STE was less accurate than TAO in $60\%$ of the benchmarks (by $5\%$ on average, worst case by $30\%$), \dt Top-$k$ was more accurate than TAO in $80\%$ of the benchmarks (by $3\%$ on average, best case by $6\%$). These findings are further strengthened by an ablation study that analyzes which component of our annealed Top-$k$ process has the greatest impact on training stability and performance compared to using the STE approximation.
Finally, Section~\ref{results-rl} extends our contribution to the RL setting, where we show the effectiveness of Top-$k$ vs. STE in continuous-action environments. Our source code for the introduced Top-$k$ approach is publicly available on GitHub under the branch ``\texttt{topk}'', which succeeds the earlier version of our code: \href{https://github.com/CPS-research-group/dtsemnet/tree/topk}{\texttt{https://github.com/CPS-research-group/dtsemnet/tree/topk}}.

We organize the paper as follows. Section~\ref{related_work} reviews related work. Section~\ref{oblique_dt} introduces the \dt architecture and its application to regression, classification, and RL tasks. Section~\ref{sec:results} presents extensive empirical evaluations across these benchmarks. Finally, Section~\ref{sec:conslusion} concludes the paper with a summary of contributions and future research directions.


\section{Related Work}\label{related_work}

\noindent
\textbf{Non-Gradient-Based DT Training.} Historically, DT training techniques have not relied on gradient descent. CART \citep{breiman1984cart} (and its extensions) is a well-known method for training {\em axis-aligned} DTs by recursively splitting the dataset at each node based on selected features using metrics such as entropy or Gini impurity. For learning {\em oblique} DTs in this manner, methods such as Oblique Classifier 1 (OC1) \citep{murthy1993oc1} and GUIDE \citep{loh2014fifty} have been proposed, but they typically yield suboptimal performance because learning oblique DTs is considerably more challenging than learning axis-aligned ones. TAO is currently a state-of-the-art method for oblique DT learning, improving DTs obtained from CART (or random DTs of a given depth) by alternately fine-tuning (using gradients but not end-to-end) node parameters at specific depths \citep{carreira2018alternating, tao-regression}.

Concerning MIP formulations \citep{bennett1996optimal, bertsimas2017optimal} 
or Global EA-based search approaches, such as CRO-DT \citep{ea2023}, they learn DTs by searching over various structures of DTs but at high computational costs, which is impractical for large DTs and datasets. 
The proposed \dt overcomes these challenges by using gradient-descent to lower training time, which we confirm by comparing with CRO-DT \citep{ea2023}, which proposes matrix encoding of (oblique) DTs to speed up training compared to previous EA-based methods, and produces axis-aligned DTs.

\smallskip
\noindent
\textbf{Gradient-Based DT Training.}
Lately, several works have proposed approximating DTs as soft DTs to enable gradient descent for learning, where decision nodes typically use the \texttt{Sigmoid} function \citep{zantedeschi2021learning,hazimeh2020tree,yang2018deep,frosst2017distilling,tanno2019adaptive,biau2019neural,silva2020_softDT,ding2021cdt,qiu2021programmatic,policy_tree,neural_forestICCV,samuel,wan2020nbdt}. {\em Hardening} soft DTs, 
that is, transforming them into hard DTs by discretizing the probabilities induces severe inaccuracies \citep{icct2022}.
More closely related to our work, DGT \citep{dgt2022} represents (oblique) DTs as an NN-architecture using the (non-differentiable) sign activation function, resorting to {\em quantized} gradient descent to learn it, leveraging principles from training binarized NNs using STE \citep{hubara2016binarized}. 
Specifically, during forward propagation, nodes utilize a 0-1 step function, whereas, during backward propagation, nodes employ a piecewise linear function or some differentiable approximation \citep[see][]{dgt2022}. Similarly, ICCT \citep{icct2022} and GradTree \citep{marton2024gradtree} learns (axis-aligned) DTs using NNs with the \texttt{Sigmoid} activation function and STEs.
In all these works, the hard DTs that are produced are (slightly) different from the DT (soft DT or using STE), which is being optimized by gradient descent. In contrast, the \dt architecture using ReLU activation functions allows standard gradient descent to be performed, and the output DT from \dt-classification is exactly the same as the function optimized by gradient-descent, without approximation. Experiments confirm that it is more accurate in practice, significantly so for classification tasks.

\smallskip
\noindent
\textbf{Top-$k$ Routing.}
DT regression bears strong similarity to MoE architectures \citep{shazeer2017outrageously, muqeethsoft}, where each regressor at a leaf can be viewed as an expert and the DT acts as a routing network that selects which leaf regressor (expert) to use, with the regressor’s output serving as the expert output. While standard MoE models typically use a Top-$k$ selection mechanism with a fixed $k \geq 2$, DT regression requires a Top-1 selection to choose a single leaf regressor for a given input. However, Top-1 selection produces flat gradients, making it difficult to train the DT effectively. Although Switch Transformer \citep{fedus2021switch} also uses Top-1 expert routing, it scales the expert output by the routing probability, which is incompatible with DT regression where the regression output must come directly from the selected leaf regressor.
To address these challenges, we instead use an annealing strategy that starts with $k \geq 2$ to provide meaningful gradient signals for learning DT parameters, which is consistent with prior findings that Top-$k$ selection yields better gradient signals than STE-based approximations \citep{muqeethsoft}, and gradually reduces $k$ to $1$ to obtain a hard regression DT. Finally, while REINFORCE \citep{muqeethsoft} offers an alternative for discrete expert selection, its high variance and weak learning signals make it ineffective in supervised regression settings.

\smallskip
\noindent
\textbf{Training DTs as Programmatic Policies in RL.}
Hard or soft DT policies can be obtained via imitation learning \citep{imitationlearning}, i.e., learning from expert policies, usually pretrained NNs \citep{bastani2018verifiable,jhunjhunwala,liu2019toward,bewley2021tripletree,roth2019conservative,verma2018programmatically,verma2019imitation,coppens2019distilling}. For instance, VIPER \citep{bastani2018verifiable} imitates a Q-network (or policy network) by creating a dataset from collected samples and trains a DT using CART, with sample weightage assigned based on Q-values. In contrast, \dt directly learns a hard oblique DT in RL (using Proximal Policy Optimization (PPO) \citep{schulman2017proximal} or SAC). 
Other works, such as ProLoNet \citep{prolonets}, that learn {\em soft} DTs using the RL framework, with the objective of initializing weights from expert humans. In contrast, \dt learns a {\em hard} DT. 
ICCT \citep{icct2022} and GradTree \citep{marton2025mitigating} proposed an STE-based approach to learn axis-aligned DTs using gradient descent. By comparison, we can handle oblique trees, which are more expressive and accurate than axis-aligned DTs, especially for discrete actions.

\section{Differentiable Oblique Decision Trees}\label{oblique_dt}

\noindent
\textbf{Notations:} We use lowercase letters (e.g., $y$) for scalars, bold lowercase (e.g., $\vx$) for vectors, and bold uppercase (e.g., $\tA$) for matrices or higher-order tensors. Blackboard bold (e.g., $\sR$) is used for sets and calligraphic font (e.g., $\mathcal{L}$) denotes functions. A dataset is represented as $\sD = {(\vx_i, y_i)}_{i=1}^N$, where inputs $\vx_i \in \sR^d$ have $d$ features and labels $y_i \in \sR$. The loss function $\mathcal{L}(\hat{y}, y)$ measures the discrepancy between prediction $\hat{y}$ and ground truth $y$. The dot product of vectors $\vu, \vv \in \mathbb{R}^d$ is denoted by $\langle \vu, \vv \rangle$.
\subsection{Oblique Decision Trees}
An \emph{oblique decision tree} \(\mathcal{T}\) comprises $m$ internal nodes $\sI = \{I_0, \dots, I_{m-1}\}$ and $n$ leaf nodes $\sL = \{L_0, \dots, L_{n-1}\}$. Each internal node performs a binary test (true or false) based on an oblique hyperplane, defined as a linear combination of input features, as illustrated in Fig. \ref{fig:tree}.
Given an input-output pair $(\vx, y) \sim \sD$, the routing function $\mathcal{T}(\vx; \tA, \bm{b}): \sR^d \to \sL$ deterministically maps $\vx$ to a leaf through a sequence of binary tests. Specifically, at node $I_j \in \sI$, the test $\langle \tA_j, \vx \rangle + \bm{b_j} > 0$ determines whether to proceed to the right child (true) or the left child (false), ignoring ties\footnote{We assume inputs almost surely do not lie exactly on decision boundaries.}, where $\tA_j \in \sR^d$ and $\bm{b_j} \in \sR$ are parameters at that internal node.
Once a leaf $L_\ell = \mathcal{T}(\vx)$ is reached, a prediction is made based on the assigned class label or model (e.g., regressor) associated with that leaf, parametrized by $\bm{\Theta}$, which collectively denotes parameters of all leaves. The training objective is to optimize $(\tA, \vb, \bm{\Theta})$ by minimizing the expected loss:
$
\texttt{min}_{\tA, \vb, \bm{\Theta}} \; \mathbb{E}_{(\vx, y) \sim \sD} \left[ \mathcal{L}(\hat{y}, y) \right].
$

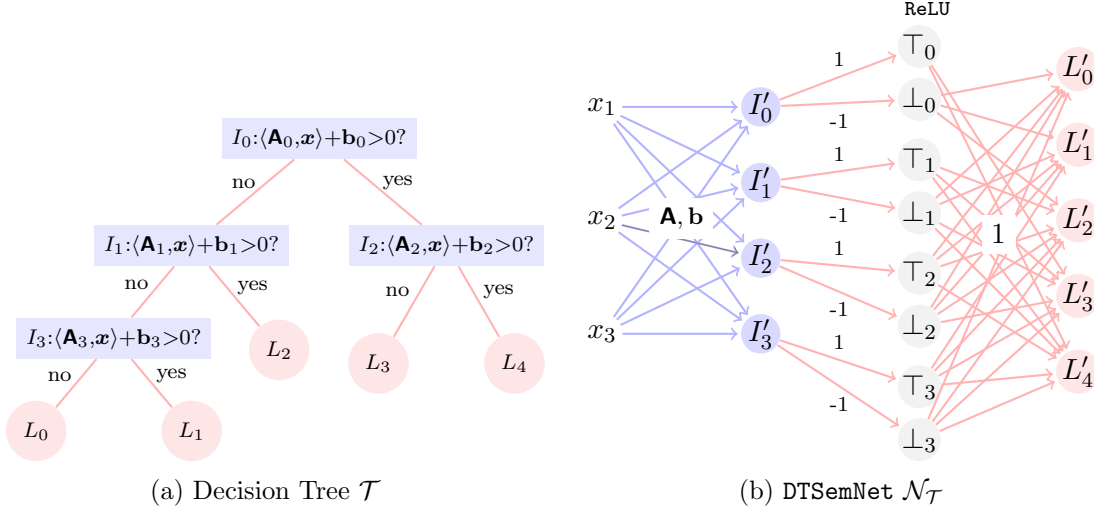
\begin{figure}[t]
  \centering

  \begin{subfigure}[b]{0.45\textwidth}
    \centering
    \begin{tikzpicture}
  [ internal/.style={fill=blue!10}, 
    leaf/.style={circle, fill=red!10, minimum size=8mm, inner sep=0pt} ]

  \node[internal] (A) {${\scriptstyle I_0: \langle \tA_0, \vx \rangle + \mathbf{b}_0 > 0?}$};
  \path (A) ++(-140:22mm) node[internal] (B) {${\scriptstyle I_1: \langle \tA_1, \vx \rangle + \mathbf{b}_1 > 0?}$};
  \path (A) ++(-40:22mm) node[internal] (C) {${\scriptstyle I_2: \langle \tA_2, \vx \rangle + \mathbf{b}_2 > 0?}$};
  \path (B) ++(-130:16mm) node[internal] (D) {${\scriptstyle I_3: \langle \tA_3, \vx \rangle + \mathbf{b}_3 > 0?}$};

  \path (D) ++(-130:16mm) node[leaf] (F) {\scriptsize $L_0$};
  \path (D) ++(-50:16mm) node[leaf] (G) {\scriptsize $L_1$};
    \path (B) ++(-50:18mm) node[leaf] (E) {\scriptsize $L_2$};
  \path (C) ++(-120:18mm) node[leaf] (H) {\scriptsize $L_3$};
  \path (C) ++(-60:18mm) node[leaf] (I) {\scriptsize $L_4$};

  \draw[red!30, thick] (A) -- (B)
  node[left, pos=0.36, text=black] {\scriptsize no};

\draw[red!30, thick] (A) -- (C)
  node[right, pos=0.36, text=black] {\scriptsize yes};

\draw[red!30, thick] (B) -- (D)
  node[left, pos=0.35, text=black] {\scriptsize no};

\draw[red!30, thick] (B) -- (E)
  node[right, pos=0.35, text=black] {\scriptsize yes};

\draw[red!30, thick] (D) -- (F)
  node[left, pos=0.35, text=black] {\scriptsize no};

\draw[red!30, thick] (D) -- (G)
  node[right, pos=0.35, text=black] {\scriptsize yes};

\draw[red!30, thick] (C) -- (H)
  node[left, pos=0.35, text=black] {\scriptsize no};

\draw[red!30, thick] (C) -- (I)
  node[right, pos=0.35, text=black] {\scriptsize yes};
\end{tikzpicture}

    \caption{Decision Tree $\mathcal{T}$}
    \label{fig:tree}
  \end{subfigure}%
  \hfill
  \begin{subfigure}[b]{0.55\textwidth}
    \centering
    \begin{tikzpicture}
  [shorten >=1pt,->,draw=black!50, node distance=\layersep]
    \tikzstyle{every pin edge}=[<-,shorten <=1pt]
    \tikzstyle{neuron}=[circle,fill=black!25,minimum size=10pt,inner sep=0pt]
    \tikzstyle{input neuron}=[neuron, fill=none, minimum size=15pt,]
    \tikzstyle{output neuron}=[neuron, fill=red!10, minimum size=10pt]
    \tikzstyle{node neuron}=[circle,fill=blue!15,minimum size=10pt,inner sep=0pt]
    \tikzstyle{hidden neuron}=[circle,fill=gray!10,minimum size=10pt,inner sep=0pt]
    \tikzstyle{annot} = [text width=1em, text centered]
    
    \def\layersep{2.1cm}

    \foreach \name / \y in {1,...,3}
        \node[input neuron] (I-\name) at (0, -0 cm - 1.5*\y cm) {$\scriptsize x_\y$};


    \foreach \name / \y in {0,...,3}
        \path[yshift=-0.5cm]
            node[node neuron] (HH-\name) at (\layersep,-1cm -1 * \y cm) {$I'_\y$};

    \foreach \name / \y in {0,...,3}
        \path[yshift=0.5cm]
            node[hidden neuron] (H-\name) at (2*\layersep,-1.2cm -1.5*\y cm) {$\scriptsize \top_\y$};

      \foreach \name / \y in {0,...,3}
          \path[yshift=0.5cm]
                node[hidden neuron] (B-\name) at (2*\layersep,-1.9 cm -1.5*\y cm) {$\scriptsize \bot_\y$};


    \foreach \name / \y in {0,...,4}
    \path[yshift=-1.5cm]
    node[output neuron] (O-\name) at (3*\layersep,0.5 cm - 1*\y cm) {$\scriptsize L'_\y$};

    \foreach \dest in {0,...,3}
            \draw[blue!30, thick](I-1) edge (HH-\dest);
    \foreach \dest in {0,...,3}
            \draw[blue!30, thick](I-2) edge (HH-\dest);
    \foreach \dest in {0,...,3}
            \draw[blue!30, thick](I-3) edge (HH-\dest);

        \foreach \source in {0,...,3}
      \draw[red!30, thick]
        (HH-\source)
        edge node[
          draw=none,
          fill=none,
          font=\scriptsize,
          text=black,   
          midway,
          above
        ] {1}
        (H-\source);
    
    \foreach \source in {0,...,3}
      \draw[red!30, thick]
        (HH-\source)
        edge node[
          draw=none,
          fill=none,
          font=\scriptsize,
          text=black,   
          midway,
          below
        ] {-1}
        (B-\source);
       
       \draw[red!30, thick](B-0) edge (O-0);
       \draw[red!30, thick](B-1) edge  (O-0);
       \draw[red!30, thick](B-2) edge (O-0);
       \draw[red!30, thick](H-2) edge (O-0);
       \draw[red!30, thick](B-3) edge  (O-0);

       \draw[red!30, thick](B-0) edge (O-1);
       \draw[red!30, thick](B-1) edge (O-1);
       \draw[red!30, thick](B-2) edge  (O-1);
       \draw[red!30, thick](H-2) edge  (O-1);
       \draw[red!30, thick](H-3) edge  (O-1);

       \draw[red!30, thick](B-0) edge  (O-2);
       \draw[red!30, thick](H-1) edge (O-2);
       \draw[red!30, thick](B-2) edge (O-2);
       \draw[red!30, thick](H-2) edge (O-2);
       \draw[red!30, thick](B-3) edge  (O-2);
       \draw[red!30, thick](H-3) edge  (O-2);

       \draw[red!30, thick](H-0) edge (O-3);
       \draw[red!30, thick](H-1) edge  (O-3);
       \draw[red!30, thick](B-2) edge  (O-3);
       \draw[red!30, thick](B-3) edge  (O-3);
       \draw[red!30, thick](H-3) edge  (O-3);

       \draw[red!30, thick](H-0) edge  (O-4);
       \draw[red!30, thick](B-1) edge  (O-4);
       \draw[red!30, thick](H-1) edge  (O-4);
       \draw[red!30, thick](H-2) edge  (O-4);
       \draw[red!30, thick](B-3) edge (O-4);
       \draw[red!30, thick](H-3) edge  (O-4);

       \draw[red!30, thick](B-1) edge node[draw=none,fill=white,midway,above, text=black] {$1$} (O-3);
       \draw(I-2) edge node[draw=none,fill=white,midway,above] {$\footnotesize \tA, \mathbf{b}$} (HH-2);
        \node[annot,above of=H-0, node distance=0.5cm] (hl) {\scriptsize \texttt{ReLU}};

\end{tikzpicture}
    \caption{\dt $\mathcal{N_T}$}
    \label{fig:dtsemnet}
  \end{subfigure}

  \caption{(a) DT with 4 internal nodes ($I_0$–$I_3$) and 5 leaves ($L_0$–$L_4$) and (b) \dt $\mathcal{N_T}$ corresponding to DT $\mathcal{T}$ in (a) with $\vx \in \sR^3$.}
  \label{fig:combined}
\end{figure}

\subsection{\dt Encoding}
Finding suitable values for $\tA$ and $\bm{b}$ to make the DT $\mathcal{T}$ accurate is challenging. We explain in the following how to find such values using standard gradient descent algorithms.
For that, we encode the DT $\mathcal{T}$ as a NN architecture \dt $\mathcal{N_T} : \sR^d \to \sR^n$, which is semantically equivalent with DT $\mathcal{T}$, and which can be converted back to an equivalent 
DT after learning. It is important to note that the \dt $\mathcal{N_T}$ outputs a vector of dimension $n$ equal to the number of leaves in DT $\mathcal{T}$. We now describe the \dt architecture, illustrated in Fig. \ref{fig:dtsemnet}:

\smallskip
\noindent
\textbf{Input Layer.}  
$d$ nodes, one for each dimension of the input $\vx \in \R^d$.

\smallskip
\noindent
\textbf{Second Layer (Linear Layer).}  
$m$ linear nodes $(I'_0,\ldots,I'_{m-1})$, each corresponding to an internal node $(I_0,\ldots,I_{m-1})$ of the DT $\mathcal{T}$ and computing $I'_i(\vx) = \langle \tA_i, \vx \rangle + b_i$.

\smallskip
\noindent
\textbf{Third Layer (Branching Activation).}  
$2m$ nodes, two per internal node:
\[
\top_i(\vx) := \mathtt{ReLU}(I'_i(\vx)), \qquad   
\bot_i(\vx) := \mathtt{ReLU}(-I'_i(\vx)).
\]
Exactly one of $\top_i(\vx)$ or $\bot_i(\vx)$ is positive (except on measure-zero boundaries) and $\top_i(\vx)+\bot_i(\vx)=|I'_i(\vx)|$.

\smallskip
\noindent
\textbf{Last Layer (Logits Output).}  
A linear layer with fixed (non-trainable) weights and $n$ nodes $(L'_0,\ldots,L'_{n-1})$, each corresponding to a leaf $(L_0,\ldots,L_{n-1})$ of the DT $\mathcal{T}$, is defined as
$$
L'_j(\vx) = \sum_{i=0}^{m-1} \Big(\delta^{\top}_{i,j}\,\top_i(\vx) + \delta^{\bot}_{i,j}\,\bot_i(\vx)\Big), 
\qquad \delta^{\top}_{i,j}, \delta^{\bot}_{i,j} \in \{0,1\}.
$$
The coefficients encode the tree structure: $\delta^{\top}_{i,j} = 0$ if $L_j$ is a left descendant of $I_i$ and $1$ otherwise; $\delta^{\bot}_{i,j} = 0$ if $L_j$ is a right descendant of $I_i$ and $1$ otherwise; and if $L_j$ is not a descendant of $I_i$, then $\delta^{\top}_{i,j} = \delta^{\bot}_{i,j} = 1$.

\begin{theorem}\label{th1}
Given a DT $\mathcal{T}$ and its equivalent NN representation \dt $\mathcal{N}_\mathcal{T}$ (used as a classifier), the output node (from $L'_is$) with the maximum value in $\mathcal{N}_\mathcal{T}$ corresponds to the same leaf selected by $\mathcal{T}$ for any input $\vx$. That is,
$
\forall\, \vx \in \mathbb{R}^d, \quad \mathtt{argmax}\, \mathcal{N_T}(\vx) = \mathcal{T}(\vx).
$
\end{theorem}

\begin{proof}
Fix $\vx \in \R^d$ and let $L_\ell = \mathcal{T}(\vx)$ be the leaf selected by the DT. Note that $L_\ell$ has all its associated decisions satisfied, whereas for any other leaf $L_j \neq L_\ell$, at least one associated decision is violated.
Recall the construction of the \dt $\mathcal{N_T}$, for each internal node $I'_i$, the branching activations are $\top_i(\vx)=\mathtt{ReLU}(I'_i(\vx))$ and $\bot_i(\vx)=\mathtt{ReLU}(-I'_i(\vx))$,
where exactly one of $\top_i(\vx)$ or $\bot_i(\vx)$ is strictly positive (ties occur only on a measure-zero set), and each leaf output is given by
$L'_j(\vx)=\sum_{i=0}^{m-1}\!\big(\delta^{\top}_{i,j}\,\top_i(\vx)+\delta^{\bot}_{i,j}\,\bot_i(\vx)\big)$.

For the selected leaf $L_\ell$, the coefficients $(\delta^\top_{i,\ell}, \delta^\bot_{i,\ell})$ ensure that the strictly positive branching activation is included for each path node, while both activations are included for off-path nodes. Hence,
$$
L'_\ell(\vx) = \sum_{i=0}^{m-1} \big(\top_i(\vx) + \bot_i(\vx)\big).
$$

Now consider any other leaf $L_j \neq L_\ell$. Since its path differs from that of $L_\ell$, there exists an internal node $I'_{i^\star}$ where the decision required by $L_j$ differs from the actual decision at that node. If $I'_{i^\star}(\vx) > 0$, the actual decision is right while $L_j$ requires left, so $\delta^{\top}_{i^\star,j}=0$ and the strictly positive term $\top_{i^\star}(\vx)$ is omitted. If $I'_{i^\star}(\vx) < 0$, the actual decision is left while $L_j$ requires right, so $\delta^{\bot}_{i^\star,j}=0$ and the strictly positive term $\bot_{i^\star}(\vx)$ is omitted. For all $i \neq i^\star$, the contributions coincide with those in $L'_\ell(\vx)$. Therefore,
$$
L'_j(\vx)=L'_\ell(\vx)-\begin{cases}
\top_{i^\star}(\vx), & I'_{i^\star}(\vx)>0,\\[2pt]
\bot_{i^\star}(\vx), & I'_{i^\star}(\vx)<0,
\end{cases}
$$
and hence $L'_j(\vx)<L'_\ell(\vx)$.
Thus $L'_\ell(\vx)$ is the unique maximum among $\{L'_0(\vx),\ldots,L'_{n-1}(\vx)\}$, i.e.,
$\texttt{argmax} \, \mathcal{N_T}(\vx)=L_\ell=\mathcal{T}(\vx)$.
\end{proof}

Informally, for a given input $\vx$, the selected leaf $\ell = \mathcal{T}(\vx)$ in the DT corresponds to the index of the output node with the maximum value in the output vector $\vz = \mathcal{N_T}(\vx)$, analogous to selecting a class label from the output of a NN. The \dt is differentiable due to its use of \texttt{ReLU} and linear activations. In the following sections, we describe how \dt can be used in classification and regression tasks without relying on approximations, and how it can be applied in RL setting as programmatic policies.

\subsection{\dt Classification:}\label{dt-classification}
In the \emph{classification} setting, each leaf $\ell \in \sL$ is assigned a class label $\Theta_\ell \in \sC$, where $\sC = \{C_1, \ldots, C_c\}$ denotes the set of $c$ possible classes, and $\bm{\Theta} = \{\Theta_\ell\}_{\ell=0}^{n-1}$ represents the class assignments for all leaves. The prediction is given by $\hat{y} = \Theta_{\mathcal{T}(\vx)}$.

We now describe how the \dt architecture implements this classification procedure. In \dt, the tree is encoded as a NN in which each leaf $L_\ell \in \sL$ corresponds to an output neuron $L'_\ell$. Since multiple leaves may share the same class label, we introduce an additional output layer with $c$ nodes $\{C_1, \ldots, C_c\}$, one for each class. Each output node $C_i$ receives input from all leaf nodes $\ell$ such that $\Theta_\ell = C_i$. The connection weight from $L_\ell$ to $C_i$ is set to $1$ if $\Theta_\ell = C_i$, and $0$ otherwise. A \texttt{MaxPool} operation is applied at each $C_i$ to select the maximum activation among its inputs. By Theorem~\ref{th1}, if $\ell = \mathcal{T}(\vx)$ for input $\vx$, then $L'_\ell$ of \dt has the highest output. Consequently, the class node $C_i = \Theta_\ell$ also receives the highest input, yielding the same prediction as $\mathcal{T}$ and preserving semantic equivalence.

\subsection{\dt Regression:} \label{dt-regression}

Traditional regression trees, such as CART~\citep{breiman1984cart} and DGT~\citep{dgt2022}, assign to each leaf $\ell$ a scalar parameter $\alpha_\ell \in \mathbb{R}$. In CART, $\alpha_\ell$ typically represents the average of training labels reaching leaf $\ell = \mathcal{T}(\vx)$, which limits generalization, even for simple linear functions.
In contrast, more expressive regression trees, such as those used in TAO-linear~\citep{tao-regression} and ICCT~\citep{icct2022}, associate a linear regressor with each leaf $\ell$, parameterized by a weight vector $\bm{\theta}_\ell \in \mathbb{R}^d$ and a bias term $\alpha_\ell \in \mathbb{R}$. The collection of linear regressor parameters across all leaves of the DT is represented by the tensor $\bm{\Theta} = [(\bm{\theta}_0, \alpha_0), \dots, (\bm{\theta}_{n-1}, \alpha_{n-1})]^\top \in \mathbb{R}^{n \times (d+1)}$, and the prediction for input $\vx$ is given by $\hat{y} = \langle \bm{\theta}_{\mathcal{T}(\vx)}, \vx \rangle + \alpha_{\mathcal{T}(\vx)}$, corresponding to a leaf-specific linear transformation over the input features. DTs with linear regressors actually correspond to a piecewise linear regression, where the different linear pieces are described by the parameters at the leaves, and the switch between pieces by the decisions of the DT.
We choose the more expressive model with linear regressors at the leaves for regression tasks with \dt.

We now study different methods to learn \dt for regression tasks, in order to support aforementioned expressive formulation with linear regressors at each leaf.
Mainly, the problem is how to jointly learn (a) the regression parameters $\bm{\Theta}$ at the leaves of the DT, {\em and} (b) the decision parameters $(\tA,\bm{b})$ at the internal nodes of the DT that route inputs to the appropriate leaf.

\subsubsection{Hardmax Selector with One-Hot Vector}
The \dt $\mathcal{N_T}$ (corresponding to DT $\mathcal{T}$) outputs logits $\vz = \mathcal{N_T}(\mathbf{x}) \in \mathbb{R}^n$ over the $n$ leaves. 
We define a selector function $\mathcal{S} : \sR^n \to \sR^n$ that transforms the logits into a selection vector.
The standard approach defines a \textit{hardmax} selector function $\mathcal{S_H}$ that maps the logit vector $\mathbf{z}$ to a one-hot vector $\vh = \mathcal{S_H}(\mathbf{z})$, where $\vh_i = 1$ iff $i = \texttt{argmax}_j z_j$, and $\vh_i = 0$ otherwise.

To encode the linear regressors at the leaves of the DT, we define function $\mathcal{R}(\vx) : \mathbb{R}^d \to \mathbb{R}^n$ that outputs $\mathcal{R}(\vx) = [\langle \bm{\theta}_0, \vx \rangle + \alpha_0, \dots, \langle \bm{\theta}_{n-1}, \vx \rangle + \alpha_{n-1}]^\top$. The final prediction is then $\hat{y} = \langle \vh, \mathcal{R}(\vx) \rangle$,which corresponds to selecting leaf (one-hot vector) and applying its associated linear regressor (regressor value corresponding to the one-hot index) and is equivalent to $\langle \bm{\theta}_{\mathcal{T}(\vx)}, \vx \rangle + \alpha_{\mathcal{T}(\vx)}$. That is, in the forward propagation phase, the one-hot vector based on the \textit{hardmax} selector incurs {\em no approximation}. 
We now turn to the gradient derivation used in the backward propagation phase.

\subsubsection{Gradient Derivation with a One-Hot Vector}

We summarize the previously-described forward propagation with one-hot vector $\vh$, producing scalar $\hat{y}$ (the regression output), as:
$$
\vx \xrightarrow{\mathcal{N_T}} \bm{z} \xrightarrow{\mathcal{S_H}} \vh \xrightarrow{ \langle \vh, \vr \rangle} \vh^\top \cdot \mathcal{R}(\vx) = \hat{y}
$$

Given a loss $\mathcal{L}(\hat{y}, y)$, and 
denoting  $\vr = \mathcal{R}(\vx)$,
the gradients with respect to the parameters (weights) \( \tA, \bm{b}, \bm{\Theta} \) is computed using the chain rule:
$$
\nabla_{\tA} \mathcal{L} = \frac{\partial \mathcal{L}}{\partial \hat{y}} \cdot \frac{\partial \hat{y}}{\partial \vh} \cdot \frac{\partial \vh}{\partial \bm{z}} \cdot \frac{\partial \bm{z}}{\partial \tA}, \quad
\nabla_{\bm{b}} \mathcal{L} = \frac{\partial \mathcal{L}}{\partial \hat{y}} \cdot \frac{\partial \hat{y}}{\partial \vh} \cdot \frac{\partial \vh}{\partial \bm{z}} \cdot \frac{\partial \bm{z}}{\partial \bm{b}}, \quad
\nabla_{\bm{\Theta}} \mathcal{L} = \frac{\partial \mathcal{L}}{\partial \hat{y}} \cdot \frac{\partial \hat{y}}{\partial \vr} \cdot \frac{\partial \vr}{\partial \bm{\Theta}}
$$

The gradient updates to the regressor parameters $\bm{\Theta}$ are given by $\nabla_{\boldsymbol{\Theta}} \mathcal{L} = \mathcal{L}'(\hat{y}, y) \cdot \vh \cdot \vx^\top$. Note that, because of the one-hot vector $\vh$, only the parameters $\bm\theta_i$ (and $\alpha_i$) corresponding to the selected leaf receive gradient updates: only the leaf regressor selected by $\vh$ is updated.

The gradient update for the (routing in the tree) parameters $\tA$ follows the chain rule:
$\nabla_{\tA} \mathcal{L} = \mathcal{L}'(\hat{y}, y) \cdot \vr^\top \cdot \frac{\partial \vh}{\partial \bm{z}} \cdot \frac{\partial \bm{z}}{\partial \tA}$ (and similarly for $\bm{b}$).
Since \dt is differentiable, $\frac{\partial \vz}{\partial \tA}$ exists. However, when using hardmax selector $\mathcal{S_H}$, the term $\frac{\partial \vh}{\partial \vz}$ is zero almost everywhere, making the gradient non-informative. 

To address this issue, the standard technique is to rely on
STE \citep{bengio2013estimating,hubara2016binarized} methods to enable gradient flow, e.g. in ICCT \citep{icct2022}, DGT \citep{dgt2022} and GradTree \citep{marton2024gradtree}, as well as in the earlier conference version of our paper \citep{Panda2024}. In the following, we elaborate on the STE mechanism.

\subsubsection{Straight-Through Estimator (STE)}

STE \citep{bengio2013estimating} enables gradient-based optimization through the hardmax selector function $\vh = \mathcal{S_H}(\vz)$ which is non-differentiable. It does so by keeping the original function unchanged during the forward pass, while approximating its Jacobian during the backward pass as $\frac{\partial \vh}{\partial \vz} \approx \mathbf{I}$ (identity function). This approximation allows gradients to flow through the discrete selector, resulting in the approximate gradient with respect to the internal tree parameters $\tA$ (similarly for $\bm{b}$) as $\nabla_{\tA} \mathcal{L} \approx \mathcal{L}'(\hat{y}, y) \cdot \vr^\top \cdot \frac{\partial \vz}{\partial \tA}$. 
Note that the gradient updates are based on the outputs of \textit{all} regressors $\vr$, and will flow through \textit{all} the leaves and then back to \textit{all} the internal decision parameters in the tree (changing all these values), including those very far from impacting the actual (accurate) forward pass. While gradient flow back from leaves close to (e.g., one switch in decision away) the selected leaf may make sense to enable 
switching of decisions, flowing through all of them seems inappropriate.
In any case, STE introduces an objective mismatch due to the forward-backward propagation discrepancy, driving the model toward suboptimal local minima and causing training instability \citep{yinunderstanding}, which in turn hurts the accuracy of learning the parameters of the DT regressor.

The brighter side is that a single STE call (and thus approximation) 
is necessary when used in \dt (such as in the earlier conference version of this paper \citep{Panda2024}), at the output layer. For comparison, methods such as DGT and ICCT \citep{dgt2022, icct2022} rely on STE at every internal node for both classification and regression, introducing further approximations that impact learning.

\subsubsection{Top-$k$ Selection}
Conceptually, the DT part of \dt for regression tasks aims at selecting 
the leaf/regressor, analogous to expert selection in the MoE framework, where accurate routing is crucial \citep{shazeer2017outrageously, muqeethsoft}.
In the MoE setting, STE can be used, but it frequently fails to learn an effective routing or classification network \citep{muqeethsoft}, which impairs the correct routing of input samples to their corresponding experts, ultimately resulting in suboptimal performance. In the MoE context, Top-$k$ selectors , $k\geq 2$, are usually preferred \citep{shazeer2017outrageously}.

For $k \geq 1$, the Top-$k$ selector function $\mathcal{S}_k$ combines a masking operation, which retains only the $k$ largest logits 
$\vz^{(k)}$ of $\vz$ (setting the rest to $-\infty$), with a temperature-scaled \texttt{Softmax} that normalizes the masked output to sum to $1$, giving rise to Top-$k$ selector $\mathcal{S}_k$. 
Formally, we define $\mathcal{S}_k(\vz) := \sigma_\tau(\vz^{(k)})$, where \texttt{Softmax} $\sigma_\tau(\vz)_i = \exp(\vz_i / \tau) / \sum_{j=1}^n \exp(\vz_j / \tau)$, $\tau > 0$ is the temperature parameter that controls the \texttt{Softmax} sharpness. Note that the \texttt{Softmax} output for $-\infty$ is zero, so $\sigma_\tau(\vz^{(k)})$ behaves as if only the Top-$k$ entries are input to the function.

\begin{figure}[htbp]
    \centering
    \begin{subfigure}[b]{0.49\textwidth}
        \centering
        \includegraphics[trim=1.4cm 1.1cm 0.6cm 0cm, clip, width=1\textwidth]{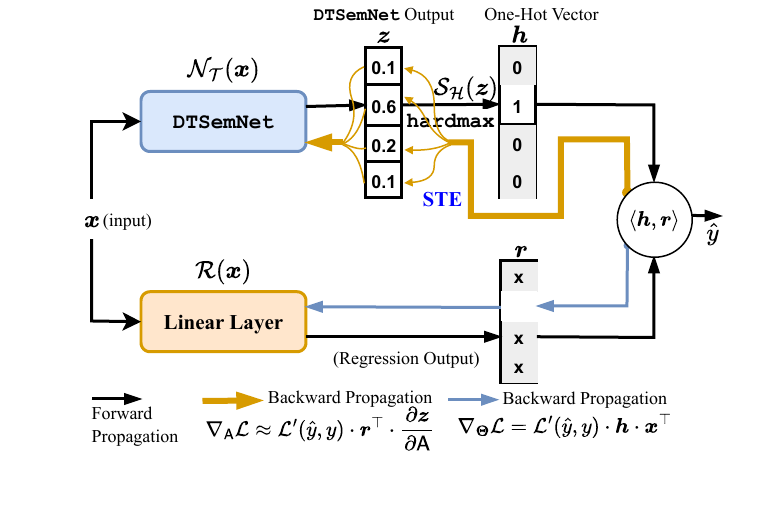}
        \caption{\dt Regression with STE}
        \label{fig:dtregnet-ste}
    \end{subfigure}
    \hfill
    \begin{subfigure}[b]{0.49\textwidth}
        \centering
        \includegraphics[trim=1.4cm 0.9cm 0.6cm 0cm, clip, width=1\textwidth]{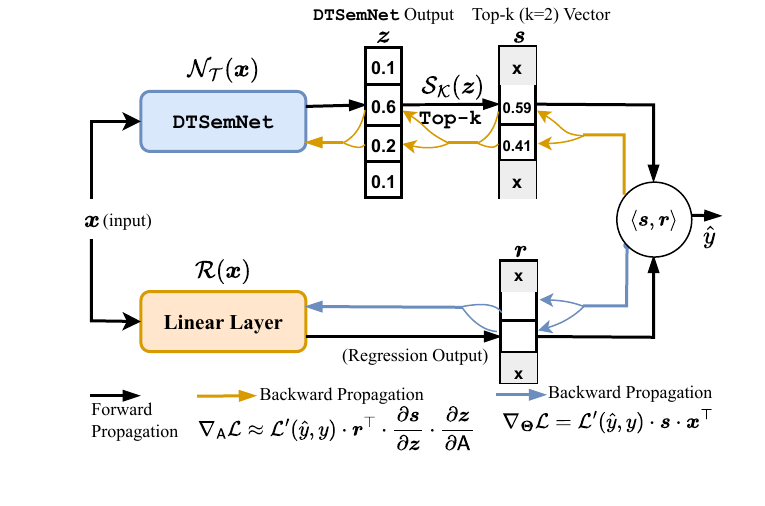}
        \caption{\dt Regression with Top-$k$}
        \label{fig:dtregnet-topk}
    \end{subfigure}
    \caption{Architecture for extending \dt to regression tasks. In (a), the thick line (orange) for backpropagation indicates that \dt is updated using all regressor outputs $\vr$, whereas in (b), only the selected leaf(s) is used for the update.}
    \label{fig:dtregnet}
\end{figure}

When used in \dt for regression tasks, we replace the one-hot vector $\vh$
by $\mathcal{S}_k(\vz)=\vs$, the final output is thus the weighted sum over the $k$ leaf-specific regressors: $\hat{y} = \vs^\top \mathcal{R}(\vx)$.
For $k=1$, the forward propagation yields the exact expected hard DT as a regressor, as $\vs=\vh$ the one hot-vector. 
For $k\geq 2$, the forward propagation is a weighted sum between the results of several ($k$) regressors at the leaves. 
The differentiability of the temperature-scaled \texttt{Softmax} in $\mathcal{S}_k$ allows gradients to flow through $\vs$, while the Top-$k$ masking operation acts as a fixed selection mechanism that preserves the computational graph thereby allowing informative gradients with $\frac{\partial \vh}{\partial \bm{z}}$. That is, there is no mismatch, and thus no approximation, between the forward propagation of Top-$k$ with the backward propagation of Top-$k$.

We compare in Fig. \ref{fig:dtregnet} the backpropagation of gradients with Top-$2$ against that with STE. There are two main differences between Top-$k$ ($k\geq 2$) and STE. First, in STE, only a single regressor is updated (the one selected by the one-hot vector). In Top-$2$, two regressors will be updated. This may be beneficial, e.g. in the case that the current routing is inaccurate, and some internal choices send the input to the wrong regressor (which will often be the case at the beginning of the learning process): then the correct regressor will also get updated with the target value, and a switch of internal choices will be more likely to happen in the future as the other regressor will be more accurate on these inputs.
Second, in STE, the outputs of {\em all} leaf regressors influence the gradient updates of {\em all} the internal node parameters of the \dt, which is extreme as nodes very far from the current decision will get updated.
By comparison, Top-$k$ routing will update the decision nodes in the path from the root to one of the Top-$k$ leaves. This will leave the decision nodes far from the actual decision for input $\vx$ unchanged by the gradient, while changing a few nearby ones to make a later switch of decision more likely (if it improves the gradient).
Additionally, the gradient of the regressor parameters $\bm{\Theta}$ is influenced by a weighted combination of the Top-$k$ leaf predictions. This soft combination distributes gradient signals across multiple leaf nodes, enabling more stable training and accurate gradient flow to both the \dt and regressors, in contrast to STE-based methods, which propagate inaccurate gradients to internal nodes of \dt.

\subsubsection{Annealing from Top-$k$ to Top-$1$}

\begin{algorithm}[t!]
    \centering
    \caption{\dt Regression with Annealed Top-$k$ Selector}
    \label{alg:dtsement-topk}
    \begin{algorithmic}[1]
    \State \textbf{Input:} Initial parameters $(\tA^0, \bm{b}^0, \bm{\Theta}^0)$, learning rate $\eta$, schedule $\mathcal{K}(t)$ specifying $k$ for Top-$k$ at each epoch $t$, dataset $(\vx, y) \sim \sD$, and total epochs $T$
    \For{$t = 0, 1, \dots, T$}
        \If{$k \geq 2$} \Comment{\textbf{Joint training}  of DT and regressors}
            \State Selector output $\vs = \mathcal{S}_{\mathcal{K}(t)}(\mathcal{N_T}(\vx))$ 
            \State Regressor output $\hat{y} = \langle \vs, \vr \rangle$ and compute loss $\mathcal{L}(\hat{y}, y)$
            \State Update parameters: $(\tA^t, \bm{b}^t, \bm{\Theta}^t) \leftarrow (\tA^{t-1}, \bm{b}^{t-1}, \bm{\Theta}^{t-1}) - \eta \nabla \mathcal{L}$
        \Else \Comment{Top-$1$: \textbf{Fine-Tuning} of  regressors}
            \State \algorithmicif\ $t < T_f$ \algorithmicthen\ Fine-tune $\bm{\Theta}$ with \textit{augmented samples} routed to the leaf
            \State \algorithmicelse\ Fine-tune $\bm{\Theta}$ with \textit{actual samples} routed to the leaf

        \EndIf       
    \EndFor
    \State \textbf{Return:} Final parameters $(\tA^T, \bm{b}^T, \bm{\Theta}^T)$
    \end{algorithmic} 
\end{algorithm}

We are interested to learn a Top-$1$ selector, as $\vs$ reduces to the one-hot vector $\vh$, and only one regressor $\bm{\theta}_\ell$ is 
called per input, the one selected by the DT, which is the expected semantics of \dt for regression tasks. 
However, we cannot learn Top-$1$ directly because the gradient is constant and non-informative during backpropagation. Instead, we propose Algorithm~\ref{alg:dtsement-topk}, which has two phases of training. It begins training \dt with Top-$k$ (first phase), $k \geq 2$ (e.g., $k=4$), softly routing each input to multiple leaves (lines 3–6). This promotes broader exploration and richer gradient signals. During training, we gradually anneal $k$ to a smaller value according to a schedule $\mathcal{K}(t)$, where $t$ is the epoch number. This sharpens the selection and induces sparser routing. The temperature parameter $\tau$ in \texttt{Softmax} further controls the sharpness of the output, where lower values make the selector more discrete, while higher values promote uniform distributions. At the end of the first training phase, we begin the second phase with $k = 1$ (lines 7–9), where only the leaf regressors are further updated.
The internal node parameters are frozen because gradients no longer propagate through \dt. This is fine as the internal node parameters have been trained before using Top-$k$, in the steps where $k\geq 2$.
For fine-tuning the regressor, we use a two-stage process. In the first stage, we train with {\em augmented samples} to combat {\em overfitting} and provide a strong initialization for the next stage. Specifically, when the training epoch is less than $T_f$, we route augmented samples to each leaf, that is, samples for which the leaf is either the first or second choice. This augmentation includes boundary samples near transition regions, thereby capturing training points not originally routed to the leaf but likely close to testing samples that will be. Training on this augmented dataset reduces overfitting and prepares the regressor with a better initialization. In the second (and final) stage, we fine-tune using only the {\em actual} samples routed to each leaf, thereby fitting the regressor more precisely to the given dataset.

This two-phase (for $k\geq2$ and $k=1$) training strategy balances exploratory learning ($k \geq 2$) during structure formation with precise fitting (leaf fine-tuning) of regressors in the final ($k=1$) phase. The Top-$k'$ step can be seen as a precise initialization of Top-$k$ step with $k < k'$.
At each step, {\em no approximation} is used, as the {\em evaluation function} used to compute the output in the forward pass matches semantically with the {\em optimization function} used to compute the gradients in the backward pass. This contrasts with STE, where the optimization function is a smoothened approximation of the evaluation function, creating a semantic mismatch.

\subsection{\dt for RL:} \label{dt-rl}
To tackle RL tasks, we use standard Deep RL frameworks (e.g., StableBaselines3 \citep{stable-baselines3} and CleanRL \citep{huang2022cleanrl}), simply replacing the NN with the \dt architecture as a programmatic policy and running gradient descent within the RL framework. For environments with discrete action spaces, we use PPO \citep{schulman2017proximal} with a \dt-classification model, assigning one class per action. For environments with continuous action spaces, we use SAC \citep{sac} with a \dt-regression model, assigning one linear regressor per action dimension (e.g., two for a continuous lunar lander: one for horizontal and one for vertical engines).
Compared with the earlier conference version \citep{Panda2024}, we compare STE vs Top-$k$ for RL with continuous action space.
To this end, we modify the SAC \citep{sac} training loop as follows. During sample collection the \dt model always uses Top-$1$, which serves as the output policy. At each update step the collected samples are used to train the model in three phases. First, training is performed with Top-$k$ selection for a fixed number of gradient steps. This is followed by training with augmented samples, and finally by Top-$1$ selection for half as many gradient steps, consistent with the procedure outlined for \dt-regression.


\subsection{Comparison with Other Encodings of DTs into NNs}

DGT \citep{dgt2022} encodes oblique decision trees using the \texttt{Sign} activation at each layer. Because \texttt{Sign} is non-differentiable (unlike the \texttt{ReLU} activation used in \dt), DGT resorts to a {\em quantized} gradient descent strategy, with an STE approximation applied at every node. In regression tasks, DGT assigns a scalar to each leaf, in contrast to the linear regressors used in \dt-regression.
The ICCT architecture \citep{icct2022} targets axis-aligned DTs for RL tasks. Each leaf is associated with the product of edge weights along its path, computed in logarithmic space to avoid explicit multiplications. A \texttt{Sigmoid} activation is applied at each decision node, yielding a soft DT approximation. During RL training, this soft DT is repeatedly “crispified” into a hard (axis-aligned) DT. To backpropagate through the resulting non-differentiable Heaviside step function, ICCT again relies on STE approximation, which \dt-classification avoids.
For \dt-regression, our earlier work \citep{Panda2024} used a single STE at the output layer, which still introduced approximation. In this work, we propose a Top-$k$ approach that eliminates this reliance, thereby offering a comprehensive framework that avoids approximation in both classification and regression settings.


\section{Experimental Evaluation}\label{sec:results}
In this section, we evaluate the performance of \dt, comparing it against competing methods for learning hard DTs. We begin with supervised learning setups, using a collection of benchmark datasets for both multi-class classification and regression. On these datasets, we compare test accuracy against state-of-the-art (SOTA) methods: TAO \citep{carreira2018alternating, tao-regression}, a non-greedy approach for learning oblique DTs; DGT \citep{dgt2022}, a gradient descent-based method also for oblique DTs; CRO-DT \citep{ea2023}, a global search-based method for axis-aligned DTs; and CART, a standard greedy algorithm for axis-aligned DTs.

We first discuss the performance of the \dt Top-$k$ approach on regression datasets, which constitutes a major contribution of this paper beyond our earlier work on \dt \citep{dtsemnet-supp}. For regression, we also analyze the effect of using Top-$k$ selection versus the STE on the learning dynamics using a synthetic dataset.

We then present results on classification datasets, along with training time comparisons. Note that training time is reported only for benchmarks where such data is available. TAO could not be evaluated in our setting due to the lack of publicly available implementation.
To further understand the learning behavior of \dt compared to DGT, we also analyze their loss landscapes following the approach of \citet{li2018visualizing}, along with a discussion on the generalization gap based on the difference between training and test accuracies.

Finally, we extend the evaluation to RL environments, considering both discrete and continuous action spaces. We compare \dt with DGT (both learning oblique DT policies via gradient descent), ICCT (learning axis-aligned DT policies via gradient descent), and VIPER, which learns axis-aligned DT policies via imitation learning from a NN policy trained with Deep RL. The performance of the underlying NN policy is also reported as a baseline.

We implemented \dt and conducted all experiments using Python and PyTorch. Our testing platform has 8 CPU cores (AMD 75F3, Zen 3 architecture), 128 GB of RAM, and a 2 GB GPU (NVIDIA Quadro P620). The supplementary material \citep{dtsemnet-supp} provides additional results and details regarding datasets, train-test splits, hyperparameters, etc. 
Our source code is available on GitHub under the new branch ``\texttt{topk}'':
\href{https://github.com/CPS-research-group/dtsemnet/tree/topk}{\texttt{https://github.com/CPS-research-group/dtsemnet/tree/topk}}.

\subsection{Regression Tasks}\label{results_regression}
We begin by evaluating the Root Mean Squared Error (RMSE) performance of the proposed regression methods on standard tabular benchmarks. Following this, we conduct an ablation study to investigate the effect of using Top-$k$ routing in \dt on learning dynamics in comparison to STE by focusing on parameter updates and overall performance. For this, we design a controlled synthetic regression task to isolate and examine how STE introduces estimation errors during the optimization of decision parameters, which in turn impairs routing accuracy and ultimately degrades regression performance.

\subsubsection{Performance on Regression Benchmark} 
We first evaluate regression performance (RMSE) on tabular benchmarks by considering five regression datasets from \citep{tao-regression}, along with ten additional datasets from \citep{grinsztajn2022tree, dgt2022}, for which TAO results and code are not publicly available (we requested access; no code was available at the time of submission). While TAO-linear \citep{tao-regression} generates DT with regressors at the leaves, similarly as \dt, allowing them to generalize, DGT \citep{dgt2022} only learns scalars at the leaves, similarly as CART, and thus is less efficient (and cannot generalize), needing deeper DTs to get acceptable accuracy. For a fair comparison, we implemented  DGT-Linear, a modification of DGT using regressors at the leaves, in the same way as in \dt and in ICCT \citep{icct2022}.
Results for the original DGT with scalars at the leaves can be found in \citep{Panda2024}, and they are not better than DGT-Linear in any case, as expected. Following \citep{tao-regression}, in Tables~\ref{tab:regression1} and~\ref{tab:regression2} we fix the height to the same (small) value when evaluating all architectures with regressors at the leaves, namely \dt Top-$k$, \dt STE, DGT-Linear, and TAO-Linear. CART is allowed to produce DTs as deep as needed, since it generates non-oblique, axis-aligned trees with scalars at the leaves, and therefore requires deep trees to achieve reasonable performance. For \dt Top-$k$, we empirically set the annealing schedule to ($4 \rightarrow 1$) with temperature $\tau = 0.5$ across all datasets.

\begin{table*}[t!]
\caption{
Average RMSE results (lower is better) on regression tasks, reported as mean $\pm$ standard deviation over 10 runs. The number of features $N_f$ and the number of training samples $N_s$ are provided for each dataset. DGT-linear is a modified version of DGT \citep{dgt2022} that we implemented to incorporate regressors at the leaves.
The datasets are taken from \citep{tao-regression}. For \dt Top-$k$, \dt STE, DGT-linear, and TAO-linear, the tree height is fixed following \citep{tao-regression}, whereas CART does not use a fixed height. CART results are taken from \citep{dgt2022}, and TAO-linear results from \citep{tao-regression}. The
\textit{\textbf{Avg $\%$ distance to best}} is computed by dividing each model RMSE by the best RMSE achieved on that dataset, then averaging these normalized values across all datasets (lower distance is better).
}
    \centering
    \resizebox{\textwidth}{!}{%
    \begin{tabular}{l c | c c c c c |c}
    \toprule
    Dataset & $N_f$, $N_s$ & \shortstack{DT \\ Height} & \shortstack{\dt \\ Top-$k$} & \shortstack{\dt \\ STE} & DGT-Linear & TAO-Linear & CART \\
    \midrule
    Abalone &10, 2004& 5 & $2.13 \pm 0.01$ & $2.14 \pm 0.03$ & $2.14 \pm 0.03$ & $\mathbf{2.07 \pm 0.01}$ & $2.29 \pm 0.034$ \\
    Comp-Active &21, 3932& 5 & $\mathbf{2.48 \pm 0.03}$ & $2.65 \pm 0.18$ & $2.65 \pm 0.15$ & $2.58 \pm 0.02$ & $3.35 \pm 0.221$ \\
    Ailerons &40, 5723& 5 & $\mathbf{1.65 \pm 0.01}$ & $1.66 \pm 0.01$ & $1.67 \pm 0.017$ & $1.74 \pm 0.01$ & $2.01 \pm 0.00$ \\
    YearPred &90, 370972& 6 & $\mathbf{8.99 \pm 0.01}$ & $\mathbf{8.99 \pm 0.01}$ & $9.02 \pm 0.025$ & $9.08 \pm 0.03$ & $9.69 \pm 0.00$ \\
    CTSlice &384, 34240& 5 & $\mathbf{1.09 \pm 0.04}$  & $1.45 \pm 0.12$ & $1.78 \pm 0.25$ & $1.16 \pm 0.02$ & $5.78 \pm 0.224$ \\
    
    \midrule
    \textit{\textbf{Avg $\%$ distance}} & \textit{\textbf{to best}} $\downarrow$ & & $\mathbf{0.6\%}$ & $8.8\%$ & $15\%$ & $3.4\%$ & $101\%$ \\
    
    \bottomrule
\end{tabular}

    }
    \label{tab:regression1}
\end{table*}

\begin{table*}[b!]
\caption{
Average RMSE results (lower is better) on regression tasks, reported as mean $\pm$ standard deviation over 10 runs. The number of features $N_f$ and the number of training samples $N_s$ are provided for each dataset. The datasets are from \citep{grinsztajn2022tree, dgt2022}, for which TAO-linear results are unavailable. For \dt Top-$k$, \dt STE, and DGT-Linear, the tree height is fixed using the best-performing height from $\{3,4,5,6\}$, whereas CART does not use a fixed height.
}
    \centering
    \resizebox{\textwidth}{!}{%
    \begin{tabular}{ll|cccc|c}
    \toprule
    Dataset & $N_f$, $N_s$ & Height & \shortstack{\dt \\ Top-$k$} & \shortstack{\dt \\ STE} & DGT-Linear & CART \\
    \midrule

    Medical &5, 114145& 5 &$\mathbf{0.84 \pm 0.01}$ & $0.85 \pm 0.02$ & $0.85 \pm 0.07$ & $1.46 \pm 1.40$ \\
    Sulfur &6, 7056& 5 & $\mathbf{0.34 \pm 0.01}$ & $0.35 \pm 0.01$  & $0.35 \pm 0.01$  & $0.43 \pm 0.03$ \\
    Bike Sharing &6, 12165& 6 & $\mathbf{1.15 \pm 0.01}$ & $1.18 \pm 0.01$ & $1.22 \pm 0.02$ & $1.22 \pm 0.04$ \\
    Houses & 8, 14448 & 5 & $\mathbf{0.28 \pm 0.00}$ & $\mathbf{0.28 \pm 0.00}$ & $\mathbf{0.28 \pm 0.01}$ & $0.45 \pm 0.01$ \\
    Wine Quality &11, 4547& 5 & $\mathbf{0.69 \pm 0.005}$ & $0.70 \pm 0.004$ & $0.70 \pm 0.005$  & $0.78 \pm 0.01$ \\
   
    Elevator &16, 11619& 4 & $\mathbf{0.20 \pm 0.00}$ & $0.21 \pm 0.00$ & $0.21 \pm 0.00$  & $0.54 \pm 0.01$ \\

    Pol &26, 10500& 5 & $\mathbf{6.61 \pm 0.34}$ & $7.62 \pm 1.38$ & $8.03 \pm 0.95$ & $11.02 \pm 1.05$ \\
    Superconduct &79, 14884& 5 & $\mathbf{1.36 \pm 0.02}$ & $1.40 \pm 0.01$ & $1.47 \pm 0.03$ & $1.51 \pm 0.01$ \\

    Microsoft &136, 578729& 5 & $\mathbf{0.77 \pm 0.00}$ & $\mathbf{0.77 \pm 0.00}$ & $\mathbf{0.77 \pm 0.00}$ & $\mathbf{0.77 \pm 0.00}$ \\

    PDBBind &2052, 9013 & 2 & $1.34 \pm 0.01$ & $\mathbf{1.33 \pm 0.01}$ & $1.34 \pm 0.01$  & $1.55 \pm 0.00$ \\

    \midrule

\textit{\textbf{Avg $\%$ distance}} & \textit{\textbf{to best}} $\downarrow$ & & $\mathbf{0.1\%}$ &  $3.1\%$ & $4.7\%$ & $50\%$ \\

    \bottomrule
\end{tabular}

    }
    \label{tab:regression2}
\end{table*}


\dt Top-$k$ demonstrates the best performance on almost all datasets compared to all other models, except in two cases: 
On the PDB-Bind dataset (Table \ref{tab:regression2}), Top-$k$ is very slightly outperformed ($<1\%$) by STE in PDBBind, within the standard deviation ($\pm 0.01$). On this test, the very limited number of leaves (four) does not discriminate much between STE and Top-$k$.
For Abalone (Table \ref{tab:regression1}), 
\dt Top-$k$ ranks second, $<3\%$ away from 
TAO-Linear. Apart from these, \dt Top-$k$ achieves the best performance on all other datasets and is statistically significantly better than all competitors across datasets, as confirmed by the Wilcoxon signed-rank test at the 0.05 significance level. 
It outperforms the state of the art TAO-Linear on 4 out of 5 benchmarks, significantly so for 3 of them: Comp-Active ($4\%$), Ailerons ($5\%$), and CTSlice ($>6\%$). On average, it produces DTs $3\%$ more accurate than TAO-Linear for regression tasks. This is a significant improvement compared to \dt STE, which was underperforming TAO-Linear in 3 out of 5 cases, including $25\%$ less accurate on CTSlice. 
Overall, \dt Top-$k$ shows an average RMSE advantage over STE of approximately $4\%$ across all datasets, with a notable $33\%$ improvement on CTSlice and $15\%$ on Pol.
The main reason is that STE underutilizes available leaves, concentrating data on fewer leaves, as illustrated in Fig.~\ref{fig:leaf-dist}. We will experimentally explain the reason why Top-$k$ learns more balanced leaves utilization than STE in Section \ref{sec:ablation}.
Concerning DGT-Linear, 
\dt Top-$k$ triples the lead (at $8\%$) that \dt STE already have over DGT-Linear, 
with a substantial $60\%$ improvement on the CTSlice dataset. 
The historical method from CART produces particularly inaccurate DTs ($>50\%$ less accurate on average), in particular 5 times less accurate for CTSlice.
Overall, \dt Top-$k$ ranks very high across all datasets, highlighting the benefits of avoiding approximations in differentiable trees.

\begin{figure}[]
    \centering
    \begin{subfigure}[b]{0.48\textwidth}
        \centering
        \includegraphics[width=1\textwidth]{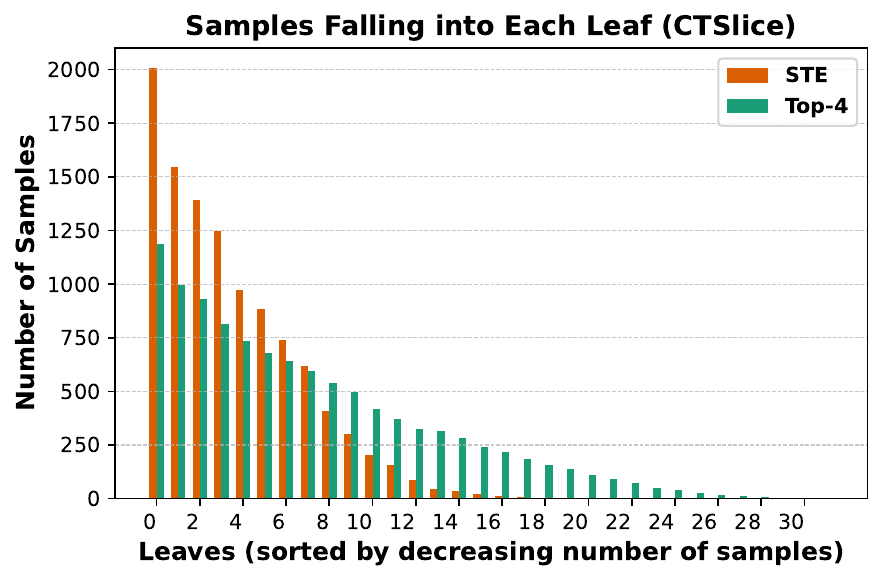}
        \label{fig:ld_ctslice}
    \end{subfigure}
    \hfill
    \centering
    \begin{subfigure}[b]{0.48\textwidth}
        \centering
        \includegraphics[width=1\textwidth]{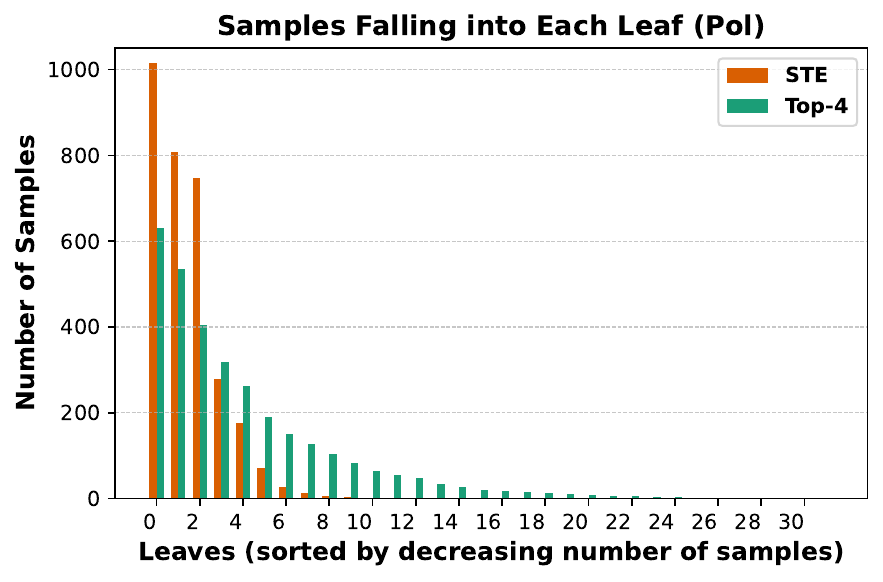}
        \label{fig:ld_pol}
    \end{subfigure}
    
    \caption{Sample distribution across leaves for the CTSlice and Pol datasets, showing that Top-$k$ ($k=4$) achieves more balanced leaf utilization. Results are averaged over 10 seeds.}
    \label{fig:leaf-dist}
\end{figure}


\subsubsection{Ablation Studies to Understand Training with Top-$k$ vs. STE}
\label{sec:ablation}

We now perform an ablation study to understand which difference(s) between STE and Top-$k$ training (updating via the gradient of {\em one} regressor in STE vs. $k$ regressors in Top-$k$; updating via the gradient of all paths from all leaves/regressors to the DT root node in STE vs. paths from $k$ leaves/regressors to the DT root node in Top-$k$) is responsible for producing significantly more accurate DTs.
We perform two distinct ablation studies:

\begin{figure}[b!]
    \centering
    \begin{subfigure}[b]{0.45\textwidth}
        \centering
        \includegraphics[width=\textwidth]{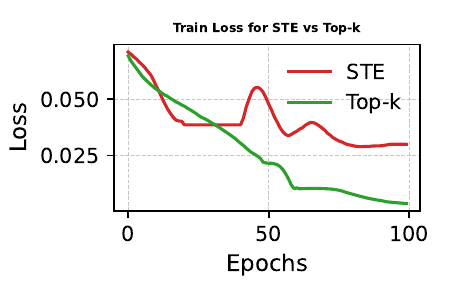}
        \caption{Train Loss}
        \label{fig:topk_vs_ste_loss}
    \end{subfigure}
    \hfill
    \begin{subfigure}[b]{0.5\textwidth}
        \centering
        \includegraphics[width=\textwidth]{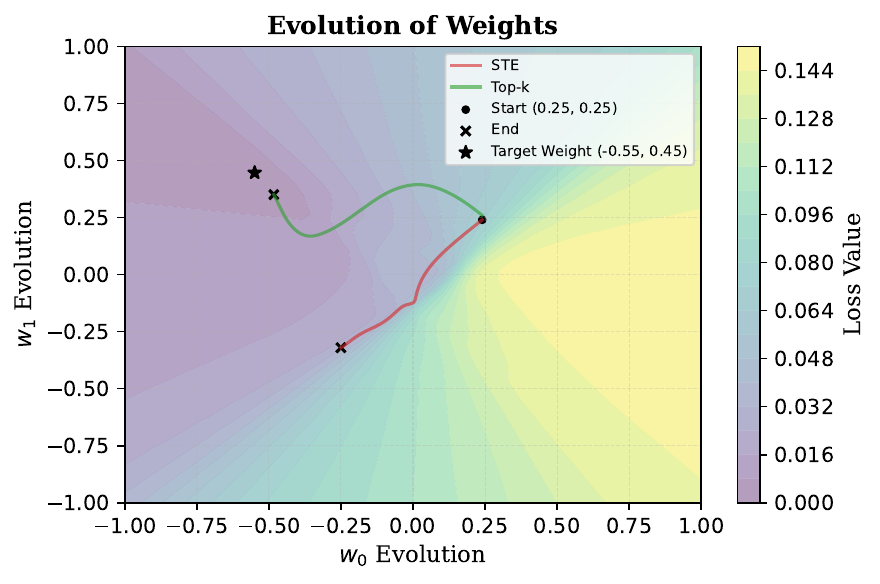}
        \caption{Evolution of weights $\vw$}
        \label{fig:topk_vs_ste_evolution}
    \end{subfigure}
    \caption{Comparison of training dynamics between \dt with Top-$k$ routing ($k=2$) and STE in a \dt-regression model of height 5, using a synthetic regression dataset generated by a teacher DT. The student network is trained to imitate the teacher.}
    \label{fig:topk-vs-ste}
\end{figure}

(a) To understand the difference between updating, during backpropagation, the internal node parameters based on all leaves/regressors (in STE) vs. only the top $k$ leaves/regressors, we design a synthetic test where only the root is updated (while other nodes and regressors are frozen). This allows us to work in low parametric dimensions (2), which can be represented graphically; and

(b) To understand the impact of updating one vs. top $k$ regressors, we cannot symmetrically freeze the DT, because training $k$ regressors is only meaningful in the early phase of learning (to enable switching in decisions), when routing in the DT is uncertain. Instead, we consider a benchmark that highlights differences in final leaves utilization between STE and Top-$k$ (see Fig.~\ref{fig:leaf-dist}). We then study how this balance evolves over time from a similar initialization, comparing STE, Top-$k$, and a hybrid Top-$k$/Reg-$1$ method, which behaves similarly to Top-$k$ except that it updates only the top-$1$ regressor during backpropagation, as in STE.

\medskip

\noindent
\textbf{(a) Training the Weight of the DT Root Node with Fixed Regressors.}

The random DT serves as the teacher model, providing the training dataset (pairs of inputs and regression values to reach) that the student models will try to imitate. The student models (STE and Top-$k$) use a \dt with the same architecture as the teacher DT (height 5, input $\vx \in \mathbb{R}^2$), but with some weights missing that need to be trained, while the others are fixed to the teacher’s weights and frozen. To ensure fairness, we initialize both STE and Top-$k$ with the same \dt\ (i.e., the missing weights are initialized in the same way). The dataset generated from the teacher comprises 30K samples, obtained by querying the teacher model with inputs uniformly sampled from the domain $[-1, 1]^2$.

All weights from the teacher network are copied to the student network, except for the weights at the root node of the DT, denoted by $\vw$. Specifically, the teacher uses $\vw = (-0.55, 0.45)$, while the student is initialized with $\vw = (0.25, 0.25)$. The objective is for the student to learn the target $\vw$ of the teacher, with all other parameters fixed and a learning rate of $0.01$.
Recall that the one-hot vector representing the semantics of \dt, equivalent to Top-$k$ with $k=1$, has no gradient backpropagated to internal nodes. To perform gradient descent on internal nodes, one must either rely on STE to create a gradient or on Top-$k$ with $k \geq 2$. We compare learning with STE and Top-$k$ ($k=2$) here, while evaluation is performed using the standard one-hot/Top-$1$ semantics of \dt.
As shown in Fig. \ref{fig:topk-vs-ste}, training with the STE converges to a suboptimal local minimum. In contrast, Top-$k$ routing leads to more stable and consistent progress toward the target weights.
This confirms the claim that the mismatch between the forward and backward objectives in STE leads to inaccurate gradient updates to the internal weights of the DT, ultimately hampering performance. 
On the other hand, there is no such problem when using Top-2 for learning the decisions but interpreting them in a Top-1 fashion.

\medskip
\noindent
\textbf{(b) Evolution of Leaves Balance with STE, Top-$k$, and Top-$k$/Reg-$1$.}

We analyze the sample distribution across leaves, i.e., the number of samples routed to each leaf, using regression benchmark datasets. For two datasets where there is a significant performance gap between STE and Top-$k$ ($\geq 10\%$), the distributions are shown in Fig.~\ref{fig:leaf-dist}. We observe that Top-$k$ ($k=4$) achieves a more balanced utilization of regressors, following a long-tailed distribution, whereas STE underutilizes many leaves.

\begin{figure}[t!]
    \centering
    \includegraphics[width=1\textwidth]{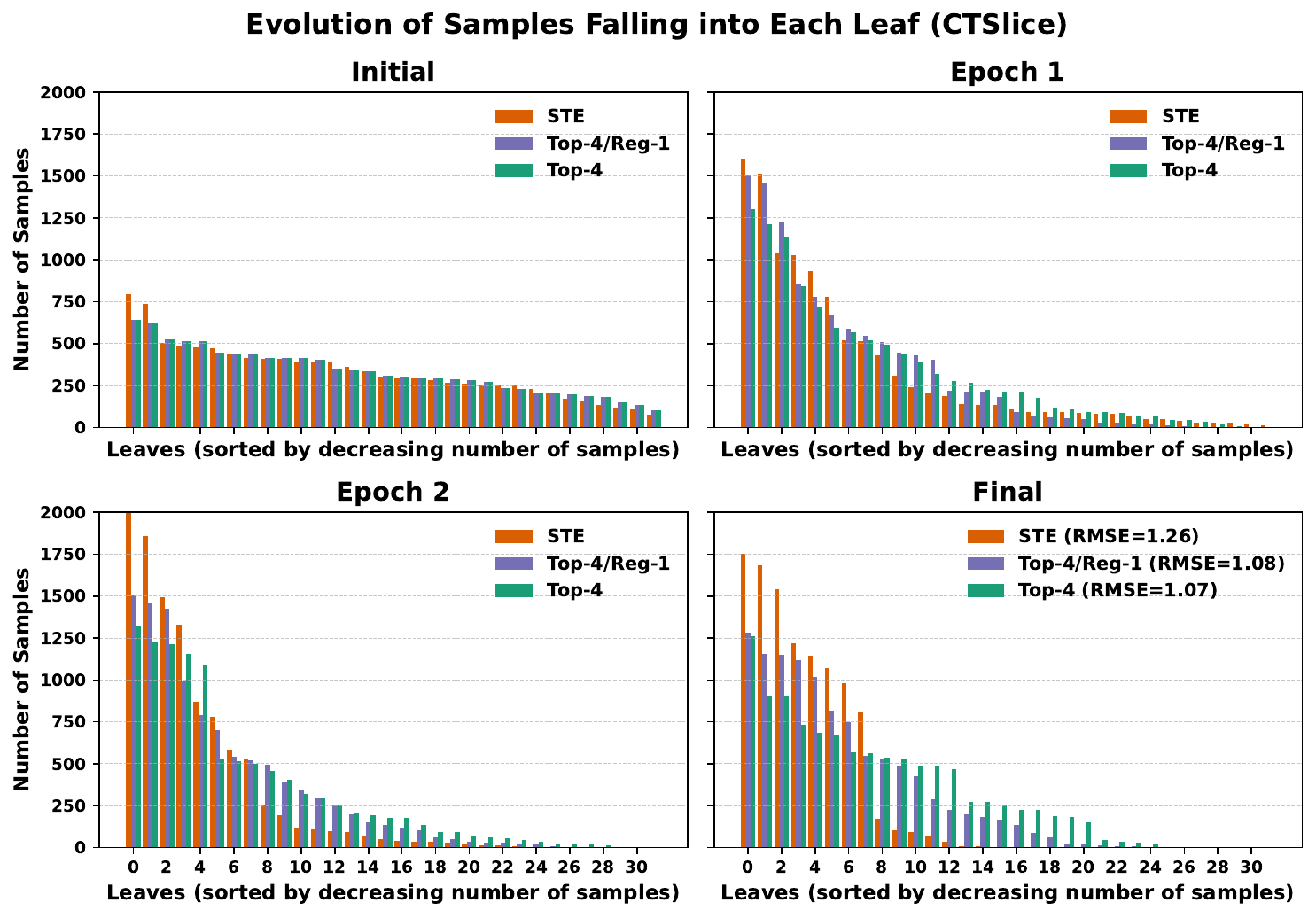}
    \caption{Sample distribution across leaves / regressors for CTSlice at the beginning of training, at epochs 1 and 2, and after training for a single seed. Initially, the sample distribution across leaves is similar for all models. Progressively, \dt with STE shows underutilization of leaves/regressors, whereas Top-$k$ ($k=4$) achieves more balanced utilization, and Top-$k$/Reg-1 being in-between.}
    \label{fig:leaf-dist-evolution}
\end{figure}

To further investigate this phenomenon (underutilization of leaves in STE versus balanced utilization in Top-$k$) we examine 
in Fig.~\ref{fig:leaf-dist-evolution} the evolution of sample distributions across leaves over training epochs on the CTSlice dataset, where the performance gap is particularly pronounced ($\geq 30\%$). 

At the beginning of training, the leaf distributions of all three models (STE, Top-4/Reg-1, and Top-4) are similar. As training progresses, Top-$k$ is observed to use leaves in a more balanced manner.

Our hypothesis is that Top-$k$, by updating multiple regressors during training, adapts more effectively to changes in decision boundaries, resulting in more balanced regressor utilization. To test this hypothesis, we modify Top-$k$ so that only one regressor is updated per sample (as in STE), which we call Top-4/Reg-1. As shown in Fig.~\ref{fig:leaf-dist-evolution}, 
STE and Top-4/Reg-1 utilize fewer regressors than Top-4 from Epoch 1 onwards.

After the first epoch, Top-4/Reg-1 uses far fewer leaves than Top-4: Top-4 utilizes 22 regressors with more than 100 samples, whereas Top-4/Reg-1 utilizes 16 regressors. This confirms that updating a single regressor has a significant impact on the balance of leaf distribution.

Interestingly, at the end of learning, Top-4 utilizes 20 regressors with more than 180 samples, compared with 15 regressors for Top-4/Reg-1 and only 8 regressors for STE, an inversion compared with Epoch 1 between Top-4/Reg-1 and STE. On top of the hypothesis we made and confirmed (that updating several regressors helps balance the regressors), there is an even stronger factor, acting over the long term, that explains why STE has a very narrow use of regressors. Because STE updates the internal nodes in the DT based on all the regressors (this is the only difference with Top-4/Reg-1, and we already proved in (a) that it was detrimental to learning a good DT), it prefers to use fewer regressors ($\approx 8$) compared to Top-4/Reg-1 to alleviate this factor.

In summary, updating DT parameters with all regressors, as in STE, has the most harmful effect, which is further compounded when updates are restricted to a single regressor. Top-$k$ avoids both issues, resulting in better performance.



\subsection{Classification Tasks} \label{results-classification}
For classification tasks, we evaluate the performance of \dt by comparing its test accuracy against other DT training approaches. We report results for both concise and dense trees, noting that concise DTs are generally more challenging to train effectively (since they have fewer choices of oblique hyperplanes to partition the input space). Prior work on CRO-DT \citep{ea2023} reports results only for relatively concise trees (height $\leq 4$). For deeper trees, CRO-DT exhibits a substantial performance drop, and since we were unable to configure the publicly available implementation to handle such depths reliably, we omit its results.
We then report the training time of the different DT learning methods to assess training efficiency.
For ablation studies, we analyze \dt and DGT to study the impact of the STE approximation on learning. Then, we discuss the performance of \dt compared to soft DTs.


\begin{table*}[b!]
    \centering
    \caption{Percentage accuracy (higher is better) on {\bf classification tasks} with the most accurate height $\leq 4$, as tested in \cite{ea2023}. For each dataset, we provide the number of features $N_f$, classes $N_c$ and training samples $N_s$. Averaged accuracy $\pm$ std is reported over 100 runs. The best-performing height is reported in parentheses for each architecture. We run the \dt and DGT experiments, while the results for TAO, CART and CRO-DT are copied from \citep{ea2023}.
    }
    \resizebox{\textwidth}{!}{%
    \begin{tabular}{llccccc}
    \toprule
    Dataset & $N_f$, $N_c$, $N_s$ & \dt & DGT & TAO & CART & CRO-DT \\
    \midrule 
    Balance Scale & 4, 3,625 & $\mathbf{90.2 \pm 2.2}$ (2) & $88.6 \pm 1.7$ (4) & $77.4 \pm 3.1$ (4) & $74.9 \pm 3.6$ (4) & $77.8 \pm 3.0$ (4) \\
    Banknote Auth & 4, 4, 1372 & $\mathbf{99.8 \pm 0.4}$ (3) & $ \mathbf{99.8 \pm 0.4}$ (3) & $96.6 \pm 1.3$ (4) & $93.6 \pm 2.2$ (4) & $95.2 \pm 2.2$ (4)\\
    Blood Transfusion & 4, 2,784 & $\mathbf{78.5 \pm 1.7}$ (2) & $78.3 \pm 2.4$ (4)& $76.9 \pm 2.1$ (3) & $77.1 \pm 1.8$ (4) & $76.1 \pm 1.9$ (3) \\
    Acute Inflam. 1 & 6, 2, 120 & $\mathbf{100 \pm 0.0}$ (2) & $\mathbf{100 \pm 0.0}$ (4) & $99.7 \pm 1.2$ (4) & $\mathbf{100 \pm 0.0}$ (3) & $\mathbf{100 \pm 0.0}$ (2)\\
    Acute Inflam. 2 & 6, 2, 120 & $\mathbf{100 \pm 0.0}$ (2) & $\mathbf{100 \pm 0.0}$ (4) & $99.0 \pm 2.6$ (2) & $99.0 \pm 2.6$ (2) & $\mathbf{100 \pm 0.0}$  (2)\\
    Car Evaluation & 6, 4, 1728 & $\mathbf{93.3 \pm 2.2}$ (4) & $92.1 \pm 2.4$ (4) & $84.5 \pm 1.5$ (4) & $84.3 \pm 1.4$ (4) & $86.1 \pm 1.3$ (4) \\
    
    Breast Cancer & 9, 2, 683 & $\mathbf{97.2 \pm 1.3}$ (2) & $\mathbf{97.2 \pm 1.2}$ (2) & $94.7\pm 1.6$ (3) & $94.7 \pm 1.7$ (3) & $95.5 \pm 1.8$ (2) \\
    Avila Bible & 10, 12, 10430 & $\mathbf{62.2 \pm 1.4}$ (4) & $59.7 \pm 1.8$ (4) & $55.8 \pm 0.8 $ (4) & $54.0 \pm 1.3$ (4) & $59.6 \pm 0.7$ (4) \\
    Wine Quality Red & 11, 6, 1599 & $ \mathbf{58.6 \pm 2.2}$ (3) & $56.6 \pm 1.4$ (4)& $56.9 \pm 2.5$ (4) & $55.9 \pm 2.3$ (4) & $55.8 \pm 2.2$ (2) \\
    Wine Quality White & 11, 7, 4898 & $\mathbf{53.5 \pm 1.4}$ (4) & $52.1 \pm 1.6$ (4) & $52.3 \pm 1.4$ (4) & $52.0 \pm 1.3$ (4) & $51.4 \pm 1.2 $ (2) \\ 
    Dry Bean & 16, 7, 13611 & $\mathbf{91.4 \pm 0.5}$ (4) & $89.0 \pm 1.6$ (4) & $83.2 \pm 1.5$ (4) & $80.5 \pm 1.9$ (4) & $77.9 \pm 4.7$ (4) \\
    Climate Crashes & 18, 2, 540 & $\mathbf{92.9 \pm 1.4}$ (2) & $92.4 \pm 2.4$ (3) & $90.6 \pm 2.2$ (3) & $91.8 \pm 1.8$ (4) & $91.5 \pm 2.0$ (2) \\
    Conn. Sonar & 60, 2, 208 & $\mathbf{82.1 \pm 5.1}$ (4) & $80.8 \pm 5.3$ (4) & $70.9 \pm 5.8$ (4) & $70.6 \pm 6.6$ (4) & $71.7 \pm 6.7$ (4) \\
    Optical Recognition & 64, 10, 3823 & $\mathbf{93.3 \pm 1.0}$ (4) & $91.9 \pm 1.0$ (4) & $64.6 \pm 6.5$ (4) & $53.2 \pm 3.2$ (4) & $65.2 \pm 2.0$ (4) \\
    \midrule
    \textit{\textbf{Avg $\%$ distance}} & \textit{\textbf{to best}} $\downarrow$ & $0.0\%$ & $1.4\%$ &  $7.4\%$ & $9.2\%$ & $7.3\%$ \\
    \bottomrule
    \end{tabular}
    }
    \label{tab:classification2}
\end{table*}

\begin{table*}[b!]
\caption{Percentage accuracy (higher is better) on {\bf classification tasks} for the datasets reported in \citep{carreira2018alternating}. 
For each dataset, we provide the number of features $N_f$, classes $N_c$, and training samples $N_s$. All methods use the tree height fixed in \citep{dgt2022} (except CART, which has no predefined height). 
Averaged accuracy $\pm$ std is reported over 10 runs. The results for DGT and CART are from \citep{dgt2022}, and the results of TAO are from \citep{carreira2018alternating}.}
\centering
\resizebox{0.9\textwidth}{!}{%
\begin{tabular}{ll|cccc|c}
\toprule
Dataset & $N_f$, $N_c$, $N_s$ & Height & \dt & DGT & TAO & CART \\
\midrule
Protein & 357, 3, 14895 & 4 & $\mathbf{68.60 \pm 0.22}$ & $67.80 \pm 0.40$ & $68.41 \pm 0.27$ & $57.53 \pm 0.00$ \\
SatImages & 36, 6, 3104 & 6 & $\mathbf{87.55 \pm 0.59}$ & $86.64 \pm 0.95$ & $87.41 \pm 0.33$ & $84.18 \pm 0.30$ \\
Segment & 19, 7, 1478 & 8 & $\mathbf{96.10 \pm 0.53}$ & $95.86 \pm 1.16$ & $95.01 \pm 0.86$ & $94.23 \pm 0.86$ \\
Pendigits & 16, 10, 5995 & 8 & $\mathbf{97.02 \pm 0.32}$ & $96.36 \pm 0.25$ & $96.08 \pm 0.34$ & $89.94 \pm 0.34$  \\
Connect4 & 126, 3, 43236 & 8 & $\mathbf{82.03 \pm 0.39}$ & $79.52 \pm 0.24$ & $81.21 \pm 0.25$ & $74.03 \pm 0.60$ \\
MNIST & 780, 10, 48000 & 8 & $\mathbf{96.16 \pm 0.14}$ & $94.00 \pm 0.36$ & $95.05 \pm 0.16 $ & $85.59 \pm 0.06$ \\
SensIT & 100, 3, 63058 & 10 & $\mathbf{84.29 \pm 0.11}$ & $83.67 \pm 0.23$ & $82.52 \pm 0.15$ & $78.31 \pm 0.00$ \\
Letter & 16, 26, 10500 & 10 & $\mathbf{89.19 \pm 0.29}$ & $86.13 \pm 0.72$ & $87.41 \pm 0.41$ & $70.13 \pm 0.08$ \\
\midrule
\textit{\textbf{Avg $\%$}} & \textit{\textbf{dist. to best}} $\downarrow$ & & $0.0\%$ & $1.6\%$ &  $1.1\%$ & $9.8\%$ \\

\bottomrule
\end{tabular}
}
\label{tab:classification1}
\end{table*}

\medskip
\noindent
\textbf{Results on Concise DTs.} 
We now turn to the 14 classification tabular datasets used in CRO-DT \citep{ea2023}. Global search methods such as CRO-DT \citep{ea2023} are efficient only for small DTs (here, up to depth 4, i.e., 32 nodes). We sort the benchmarks by the number of features, since handling more features is increasingly difficult, particularly for small DTs. Table~\ref{tab:classification2} reports the (average) score over 100 DTs trained with different seeds for the most accurate height (up to $4$). Across all benchmarks, \dt consistently produces the most accurate DTs, with the largest improvement observed on {\em Dry Beans}, where the classification error decreases from $11\%$ to $8.6\%$. DGT ranks second overall, except on the two wine quality benchmarks, where TAO outperforms it. Typically, DGT comes close to \dt, which is expected given their similarity in design, though the use of approximations (STEs, quantized gradient descent) makes DGT perform on average $0.85\%$ worse than \dt. Notably, this gap increases to $1.3\%$ on the four benchmarks with the most features (the most challenging cases). Moreover, DGT often requires larger trees than \dt to achieve its best results (in $6$ out of $14$ benchmarks). When compared with non–gradient-based methods, the advantage of \dt is striking: $5.5\%$ on average, rising to $12\%$ on the four benchmarks with the most features. The largest margin is observed in {\em Optical Recognition}, where the classification error drops from $34.8\%$ to only $6.7\%$.

\medskip
\noindent
\textbf{Results on Denser DTs.}
We then consider the 8 classification tasks originally introduced in TAO \citep{carreira2018alternating} and later used in \citep{dgt2022}. All of these involve tabular datasets, except for MNIST, which is a small image dataset.
We follow the fixed tree heights specified in \citep{carreira2018alternating} (and reused in \citep{dgt2022}). The benchmarks are sorted by tree height, which serves as a good indicator of task complexity. Table~\ref{tab:classification1} reports the (average) accuracy over 10 DTs trained with different seeds.
Consistent with the results in Table~\ref{tab:classification2} for shallower DTs, \dt achieves the highest accuracy across all benchmarks with deeper DTs as well.
With deeper trees, TAO improves and ranks second in 5 tasks, compared to 3 for DGT. On average, \dt outperforms TAO by $1\%$, with the gap increasing to $1.4\%$ on the 4 benchmarks with deeper trees. The largest differences are observed in {\em Letter}, where the classification error decreases from $12.6\%$ to $10.8\%$, and in {\em Pendigits}, from $3.9\%$ to $3\%$. Relative to DGT, \dt achieves an average improvement of $1.4\%$, which grows to $2.1\%$ on the four benchmarks with deeper trees, larger than the margin observed in smaller trees from Table~\ref{tab:classification2}. The most pronounced improvements occur in {\em Letter}, with error reduced from $13.9\%$ to $10.8\%$, and in {\em MNIST}, from $6\%$ to $3.9\%$.
Finally, a Friedman–Nemenyi test across all classification benchmarks at a significance level of $0.05$ confirms that \dt is significantly better than all other methods. The average ranks are: 1.11 for \dt, 2.25 for DGT, 2.86 for TAO, 3.60 for CRO-DT, and 3.77 for CART.


\medskip
\noindent
\textbf{Training Time.}
We report in Table \ref{tab:time} the training time for the architecture on MNIST, as \citep{carreira2018alternating} reports this number for TAO. We train other architectures on comparable computing configurations. We provide the accuracy number obtained for reference. We also report the training time for the simpler DryBean, for comparison sake, except for TAO, for which this is not available. 
First, CRO-DT for MNIST takes much longer to train (number of generations set to the default 4k) than other architectures while having very low accuracy numbers ($<60 \%$ instead of $>90\%$), the reason why we do not consider it for these benchmarks with deeper DTs (indeed, \citep{ea2023} does not report CRO-DT results for these benchmarks). The training times of DGT and \dt are almost identical due to their architectural similarities. As expected, training with gradient-based methods is observed to be significantly faster than non-gradient learning methods.

\begin{table}[h!]
\caption{Training times in seconds (lower is better), and (avg accuracy $\%$ (higher is better)). The MNIST training time from the non-publicly available TAO is quoted from \citep{carreira2018alternating}, while we train other architecture on a similar compute configuration. The tree height is 8 for MNIST and 4 for DryBean.}
\centering
\resizebox{0.55\textwidth}{!}{%
    \begin{tabular}{lllll}
    \toprule
    Dataset & \dt & DGT & TAO & CRO-DT \\
    \midrule
    MNIST & $306$ ($96.1$) & $288$ ($94.0$) & $1200$ ($95.0$) & $4659$ ($58.2$) \\
    DryBean & $4.4$ ($91.4$) & $3.8$ ($89.0$) & NA ($83.2$) & $1300$ ($77.9$) \\
    \bottomrule
    \end{tabular}
}   
\label{tab:time}
\end{table}

\medskip
\noindent
\textbf{Comparision with DGT.}
To better understand the difference in performance due to the difference in architecture used by \dt and DGT, we resort to two difference studies. 
First, the loss landscape \citep{li2018visualizing} (see Fig. \ref{fig:loss_ls}) to understand how easy 
learning is for a given architecture; and generalizability potential, by considering the difference between train and test accuracies (in Table \ref{tab:genralization}).
We restrict to the 4 datasets from Table \ref{tab:classification2} for which there is a significant difference ($>1$\%) in accuracy on the test data, namely SatImages, Connect4, MNIST and Letter. For both studies, we rerun the algorithms on the datasets, implying a small deviation in test accuracy with the results from Table \ref{tab:classification2}. The training times are reported in brackets in Table \ref{tab:genralization}.
Fig. \ref{fig:loss_ls} shows the loss landscape around the final trained parameters along two random vector directions of training parameters. 
A flatter loss landscape indicates that changing the trained parameters does not significantly impact the (trained) accuracy, leaving flexibility to adjust the parameters to accommodate updates during training without affecting the current accuracy.
On the contrary, a very steep loss landscape means that parameters cannot be changed without worsening the loss, so they must be tuned very precisely. In Fig.~\ref{fig:loss_ls}, the loss landscapes of SatImages and MNIST are much flatter for \dt than for DGT, while those of Connect4 and Letter are less flat.
The fact that \dt is easier to train than DGT on MNIST and SatImages results in better training accuracy (Table \ref{tab:genralization}). This largely explains the better {\em test} accuracy; generalizability is not improved, as computed by the difference in accuracy between train and test. The case is similar for Letter: although the loss landscape is not flat, there is still sufficient parameter freedom for \dt to achieve high training accuracy ($95\%$), much better than DGT. The last case is Connect4: here, the loss landscape is equally not flat for \dt and DGT, and the training accuracy is only average, around $85\%$. This time, \dt generalizability is superior to DGT, accounting for half of the $2\%$ improvement in test accuracy.

\begin{figure}[t!]
  \centering 
  \begin{subfigure}{0.49\textwidth}
      \centering
      \includegraphics[scale=0.22]{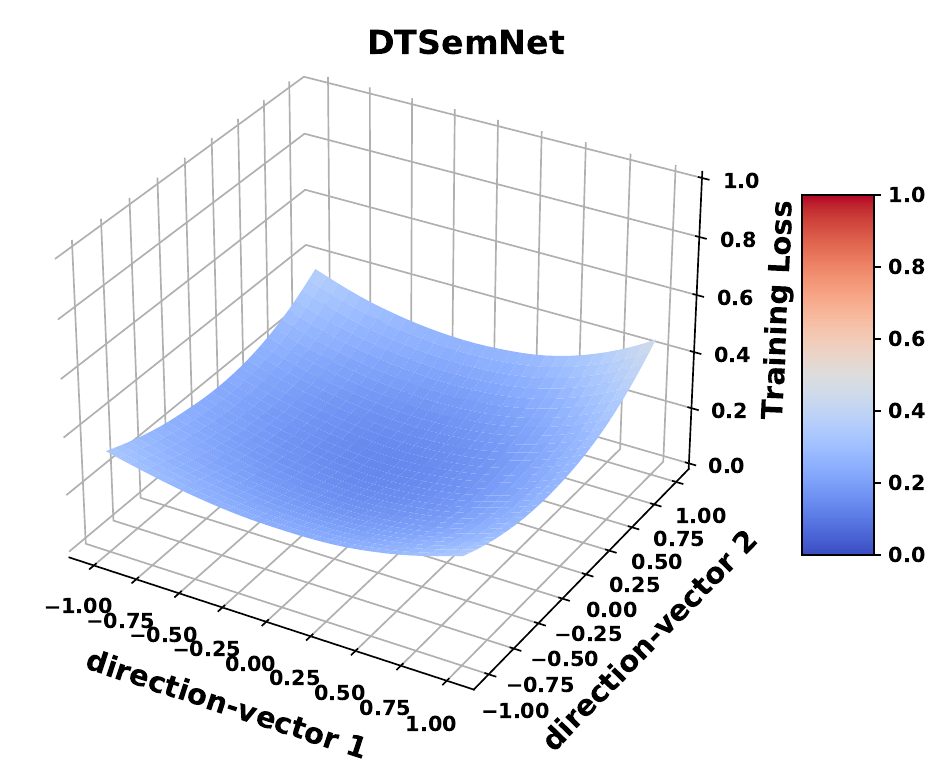}
      \includegraphics[scale=0.22]{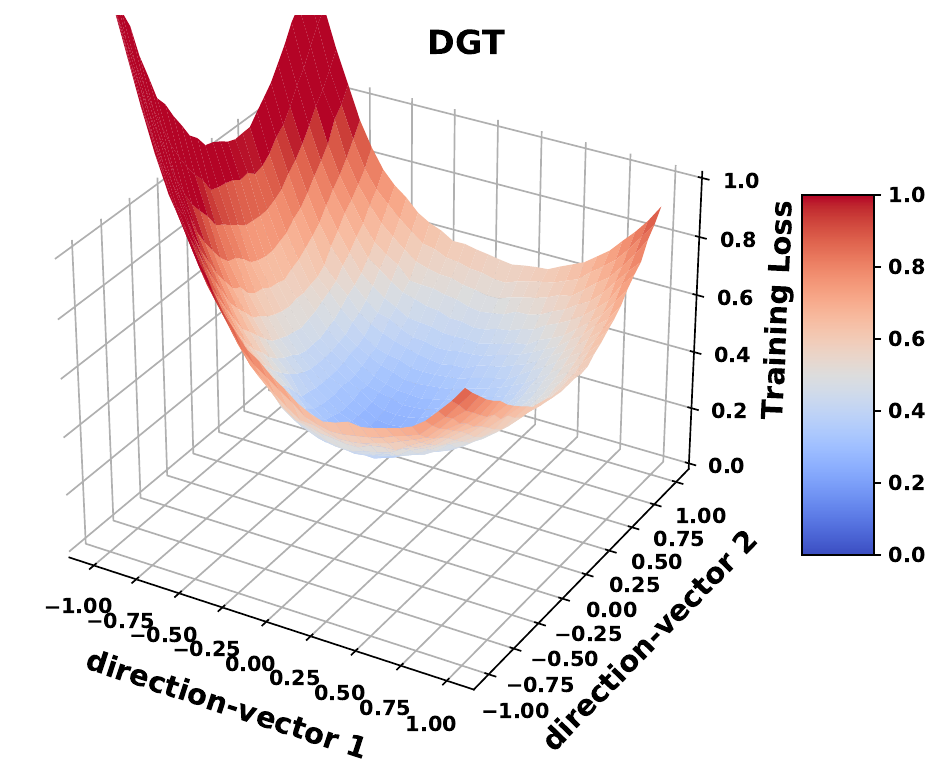}
      \caption{SatImages}
  \end{subfigure}
  \hfill
  \begin{subfigure}{0.49\textwidth}
      \centering
      \includegraphics[scale=0.22]{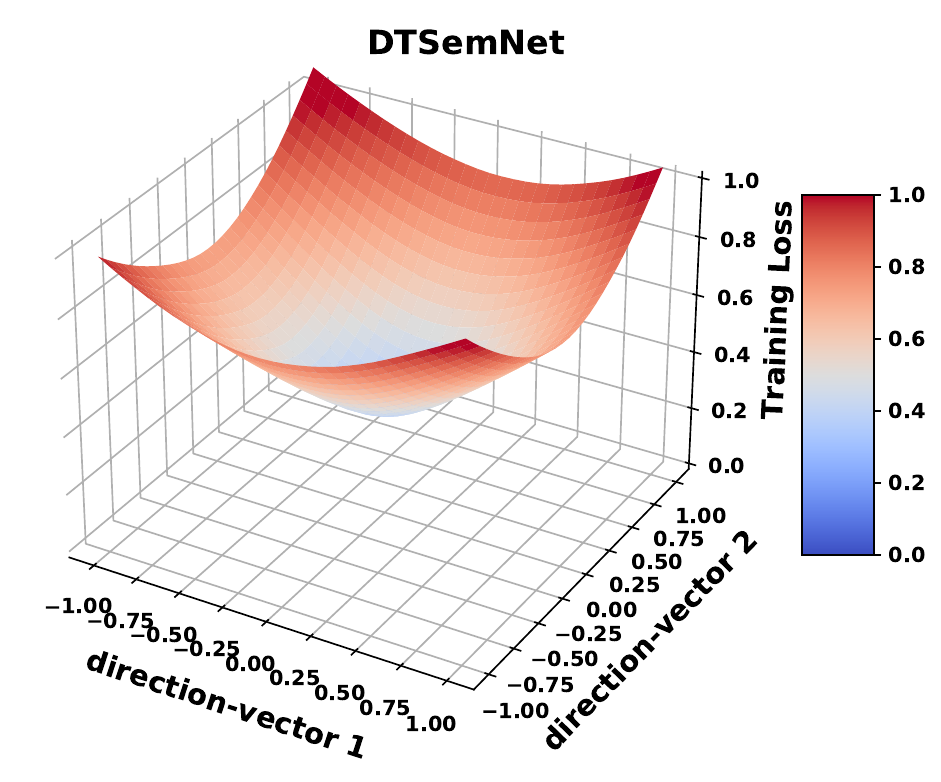}
      \includegraphics[scale=0.22]{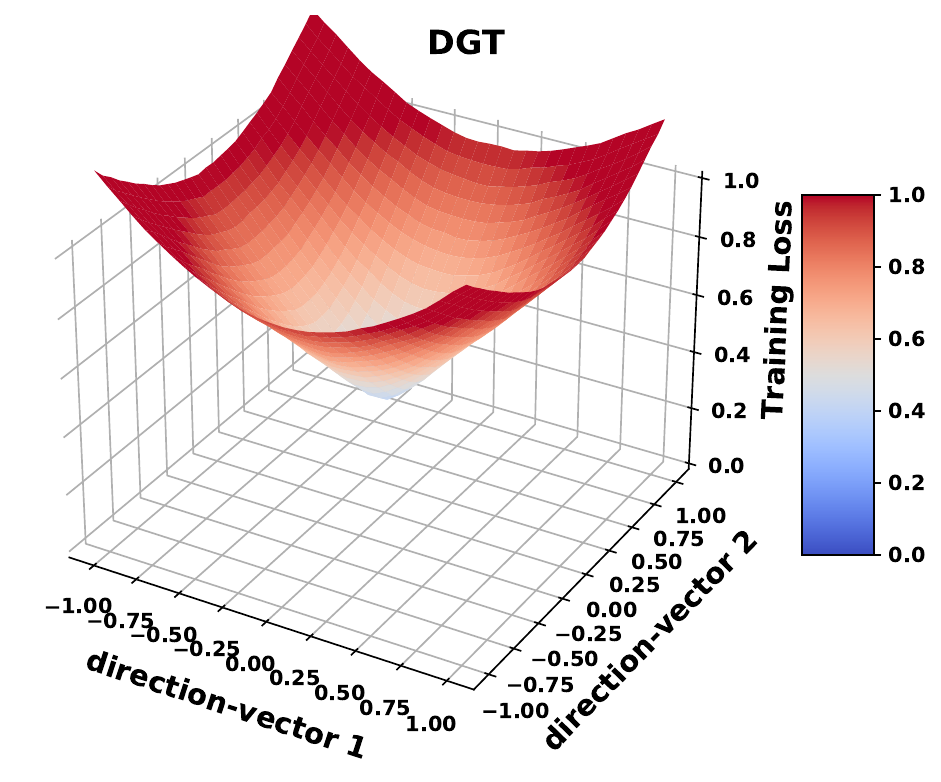}
      \caption{Connect4}
  \end{subfigure}

  \begin{subfigure}{0.49\textwidth}
      \centering
      \includegraphics[scale=0.22]{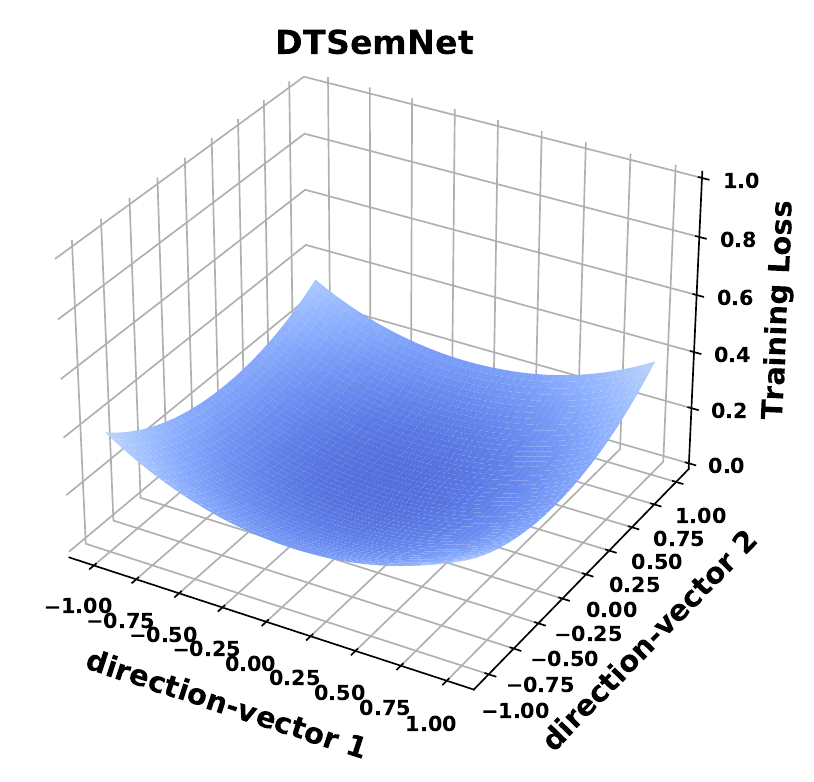}
      \includegraphics[scale=0.22]{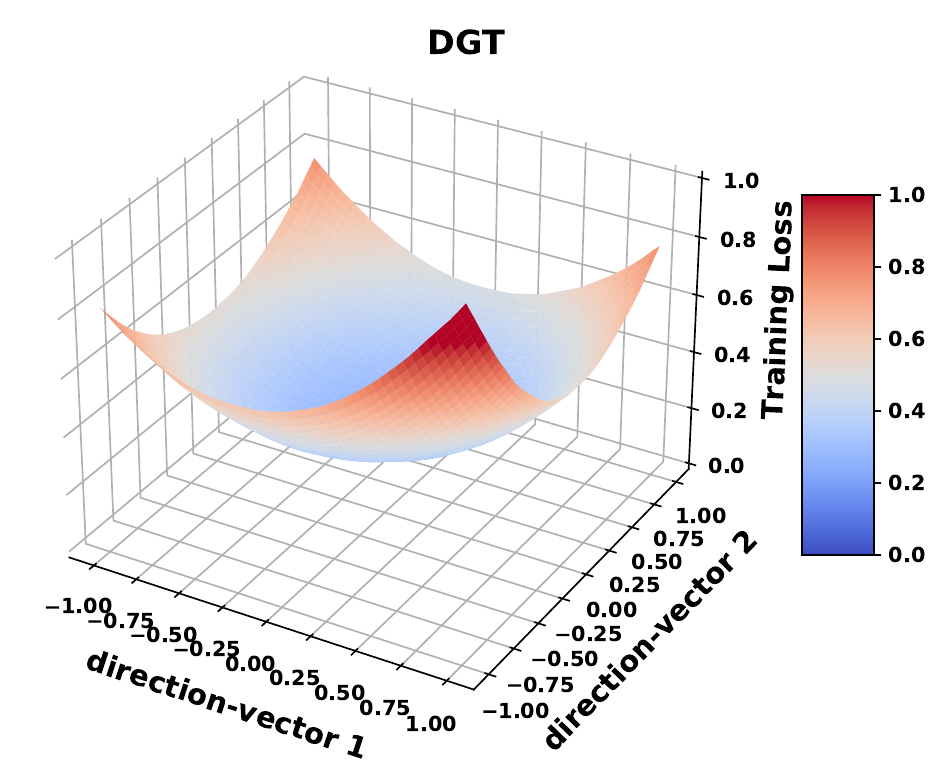}
      \caption{MNIST}
  \end{subfigure}
  \hfill
   \begin{subfigure}{0.49\textwidth}
      \centering
      \includegraphics[scale=0.22]{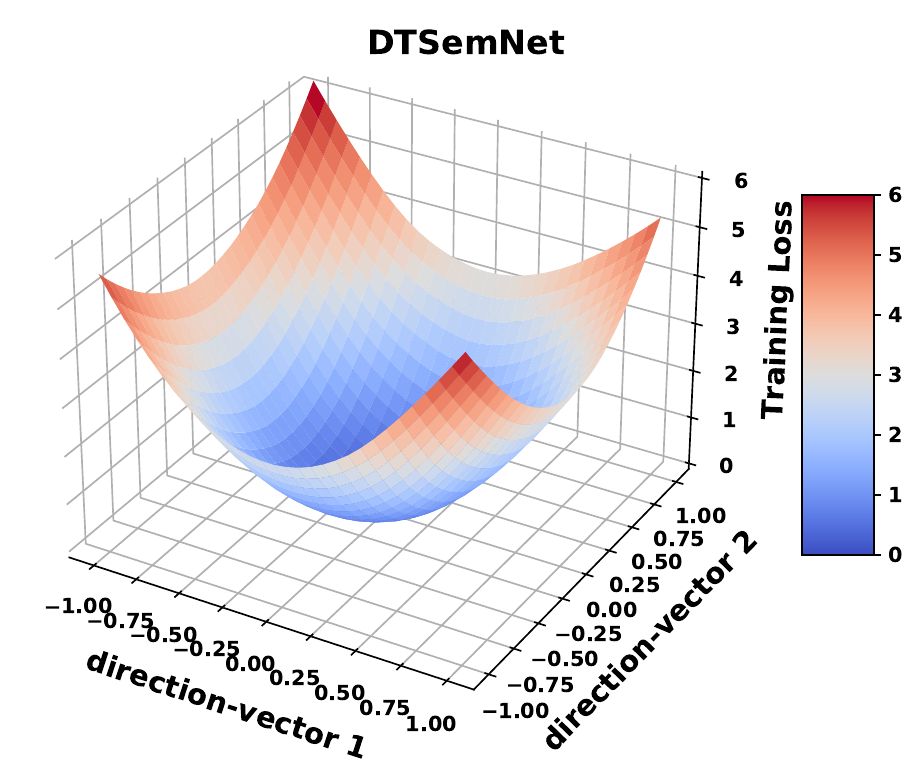}
      \includegraphics[scale=0.22]{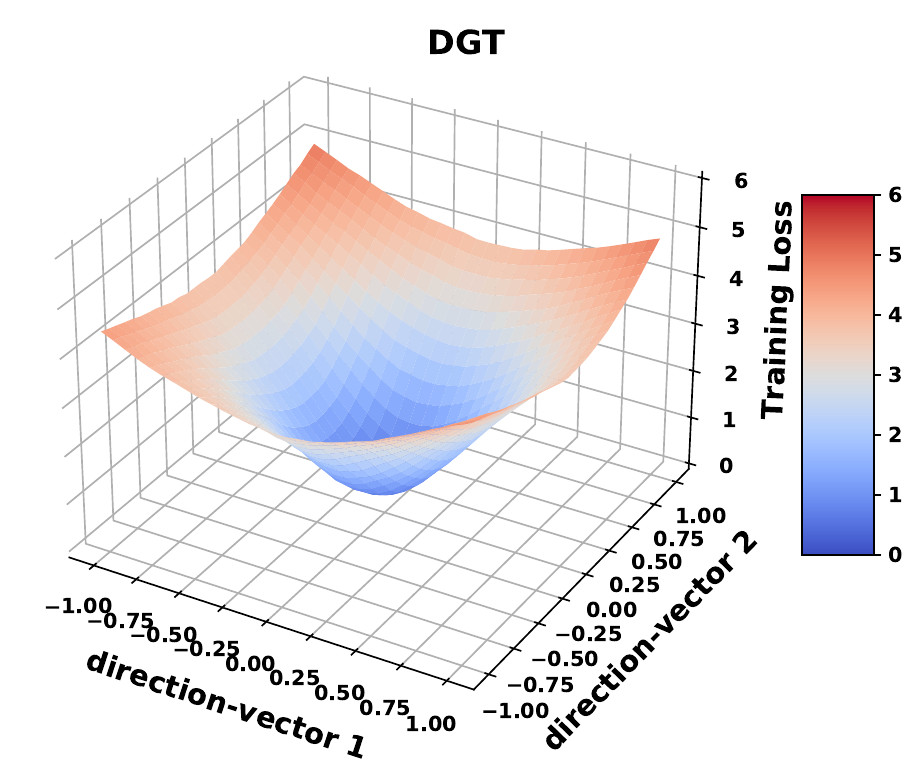}
      \caption{Letter}
  \end{subfigure}

  \caption{The loss landscape of \dt and DGT for different datasets. A flatter landscape allows flexible parameter updates without affecting training loss.}
  \label{fig:loss_ls}
\end{figure}

\begin{table*}[t!]
\caption{Performance comparison of \dt and DGT in terms of train and test accuracy. Training time is provided in brackets. The results for DGT are obtained on our machine.
}
\centering
\resizebox{0.85\textwidth}{!}{%
\begin{tabular}{l|cc|cc|cc}
\toprule
 & \multicolumn{2}{c|}{Train Accuracy} & \multicolumn{2}{c|}{Test Accuracy} & \multicolumn{2}{c}{Difference: (Train Acc. -- Test Acc.)} \\
Dataset & \dt & DGT & \dt & DGT & \dt & DGT \\
\midrule

SatImages & $\mathbf{95.30 \pm 0.60}$ (42s) & $92.22 \pm 0.80$ (25s) & $\mathbf{87.55 \pm 0.59}$ & $86.08 \pm 0.98$ & $7.75$ & $6.14$ \\

Connect4 & $\mathbf{85.86 \pm 0.35}$ (352s) & $84.78 \pm 0.36$ (598s) & $\mathbf{82.03 \pm 0.39}$ & $80.06 \pm 0.30$ & $3.83$ & $4.72$  \\

MNIST & $\mathbf{98.78 \pm 0.06}$ (306s) & $96.93 \pm 0.21$ (288s) & $\mathbf{96.16 \pm 0.14}$ & $94.37 \pm 0.22$ & $2.62$ & $2.56$ \\

Letter & $\mathbf{95.04 \pm 0.91}$ (381s) & $89.75 \pm 0.66$ (282s) & $\mathbf{89.19 \pm 0.29}$ & $84.52 \pm 0.48$ & $5.85$ & $5.23$ \\

\bottomrule
\end{tabular}
}
\label{tab:genralization}
\end{table*}

\medskip
\noindent
\textbf{Comparison with Adaptive Neural Tree.}
We compare the performance of \dt, a hard oblique DT, with Adaptive Neural Tree (ANT) \citep{tanno2019adaptive}. While ANT has a DT strucutre, it uses NNs at decision nodes, whereas \dt uses linear expressions. This is a major difference, making ANT more suitable at analyzing complex shapes (e.g. MNIST benchmark). Further, 
ANT usually produces soft (probabilistic) DTs, although it could be configured to produce hard DT (but still with NN at decision nodes), namely ANT-Hard which we compare with. While NNs increase node capacity, they negate the advantage of simple linear expressions, which are particularly appreciated for their simplicity, interpretability and amenability to verification.
In particular, for MNIST, the average number of decision nodes in the learned tree is 2, which is basically two convolutional layers (at internal nodes) followed by a 2-layer ReLU-activated classification head (at leaf nodes), explaining the similar accuracy for ANT-Soft and ANT-Hard.
As shown in Table \ref{tab:ant}, \dt compares favorably with ANT-Hard, except for MNIST for the reason described above. Still, \dt is overall more accurate than ANT-Hard ($2.7\%$ on average), despite the limited capacity. We provide numbers for ANT-Soft for reference: the performance of \dt is close to ANT-Soft, despite its lower capacity, hard DT and simpler structure.

\begin{table*}[t!]
\caption{Performance comparison with ANT from \cite{tanno2019adaptive}. Test accuracy (training time) is reported for \dt, ANT-Hard, and ANT-Soft averaged over 10 seeds. ANT-Soft is a soft DT, while ANT-Hard is obtained by hardening ANT-Soft.
}
\centering
\resizebox{0.65\textwidth}{!}{%
\begin{tabular}{l|cc|c}
\toprule
Dataset &\dt & ANT-Hard & ANT-Soft \\
\midrule

Protein & $68.60 \pm 0.22$ (43s)  & $\mathbf{68.81 \pm 0.07}$ (12s) & $68.81 \pm 0.07$  \\ 

SatImages & $\mathbf{87.55 \pm 0.59}$ (41s)  & $83.73 \pm 13.53$ (19s) & $88.15 \pm 0.58$  \\ 

Segment & $\mathbf{96.10 \pm 0.53}$ (3s)  & $91.83 \pm 12.74$ (20s)  & $96.32 \pm 1.12$  \\ 

Pendigits & $\mathbf{97.02 \pm 0.32}$ (114s)  & $95.58 \pm 3.05$ (101s)  & $96.36 \pm 0.69$  \\ 

Connect4 & $\mathbf{82.03 \pm 0.39}$ (352s) & $80.94 \pm 1.92$ (613s)  & $83.01 \pm 0.33$  \\ 

MNIST & $96.16 \pm 0.14$ (306s) & $\mathbf{98.14 \pm 0.08}$ (538s)  & $98.14 \pm 0.13$  \\ 

SensIT & $\mathbf{84.29 \pm 0.11}$ (1200s) & $76.36 \pm 10.32$ (1233s) & $84.20 \pm 0.27$  \\ 

Letter & $\mathbf{89.19 \pm 0.29}$ (381s) & $87.92 \pm 3.10$ (588s) & $91.09 \pm 0.95$  \\ 

\midrule

\textit{\textbf{Avg $\%$}} \textit{\textbf{dist. to best} $\downarrow$ } & $0.3\%$ & $3.0\%$ &  \\

\bottomrule

\end{tabular}
}
\label{tab:ant}
\end{table*}

\subsection{RL Tasks:} \label{results-rl}

\begin{table*}[b!]
\caption{Average rewards (higher is better) on {\bf RL tasks} $\pm$ std over 5 policies generated with different learning seeds. The policies are evaluated over 100 episodes. We report the number $N_f$ of features and 
$N_a$ of actions for each environment. The first four environments 
have discrete actions, and the bottom two continuous actions, for which we test both the original ({\em scalar}) version of DGT and DGT-linear ({\em linear}). For the 5 first environments, the {\em Height} is fixed for all architectures. For Bipedal Walker, DGT(-linear) performed much better with deeper DTs, while ICCT performed much better with shallower DTs (\dt was not much impacted by the height). For Bipedal Walker, the heights yielding the best rewards for DGT(-linear) and ICCT are shown in parentheses.}
  \centering
  \resizebox{0.92\textwidth}{!}{%
  \begin{tabular}{llcccccc}
   \toprule
   Environments & $N_f$, $N_a$ & {\em Height} & \dt & Deep RL & DGT & ICCT & VIPER \\ 
   \midrule
    CartPole & 4, 2 & 4 & $\mathbf{500 \pm 0}$ & $\mathbf{500 \pm 0}$ & $\mathbf{500 \pm 0}$ & $496 \pm 0.3$ & $499.95 \pm 0.05$ \\
    Acrobot & 6, 3 & 4 & $\mathbf{-82.5 \pm 1.05}$ & $-84 \pm 0.84$ & $-83.1 \pm 1.88$  & $-88.6 \pm 1.77$ & $-83.92 \pm 1.59$ \\
    LunarLander & 8, 4 & 5 & $\mathbf{252.5 \pm 3.9}$ & $245 \pm 14.5$ & $183.6 \pm 14.6$ & $-85 \pm 16.3$ & $86.73 \pm 7.93$ \\
    Zerglings & 32, 30 & 6 & $\mathbf{15.54 \pm 2.07}$ & $10.47 \pm 0.23$ & $8.21 \pm 1.03$ & $9.40 \pm 1.10$ & $10.61 \pm 0.46$ \\
    \toprule
    Cont. LunarLander & 8, 2 dim. & 4 & \begin{tabular}[l]{@{}l@{}} STE: $277.24 \pm 2.09$ \\ Top-$k$: $\mathbf{282.49 \pm 3.46}$ \end{tabular} & $276.12 \pm 1.45$ & \begin{tabular}[l]{@{}l@{}} {\em scalar}: $131.92 \pm 51.49$ \\ {\em linear}: $267.9 \pm 9.37$ \end{tabular} & $255.57 \pm 4.19$ & NA \\
    Bipedal Walker & 24, 4 dim. & 7 & \begin{tabular}[l]{@{}l@{}} STE: $314.98 \pm 3.35$ \\ Top-$k$: $\mathbf{325.35 \pm 0.79}$ \end{tabular} & $315.3 \pm 6.91$ & \begin{tabular}[l]{@{}l@{}} {\em scalar}: $78.33 \pm 57.19$ (8)\\ {\em linear}: $244.5 \pm 61.84$ (8) \end{tabular} & $301.34 \pm 3.09$ (6) & NA \\
\bottomrule
\end{tabular}
}
\label{tab:rl}
\end{table*}

We evaluate our approach on a range of RL environments. For discrete-action tasks with limited features ($\leq 32$), we consider CartPole (4 features, 2 actions), Acrobot (6 features, 3 actions), LunarLander (8 features, 4 actions), and FindandDestroyZerglings (32 features, 30 actions) \citep{starcraft}. For continuous-action tasks, we include continuous LunarLander (8 features, 2 action dimensions) and BipedalWalker (24 features, 4 actions). For all tasks, we fix the DT height across architectures and compare against Deep RL baselines with NNs matched in the number of learned parameters. We train 5 policies per environment using different seeds, applying PPO \citep{schulman2017proximal} for discrete actions and SAC \citep{sac} for continuous actions, and evaluate each on 100 test seeds to compute mean rewards, reported in Table~\ref{tab:rl}. For \dt Top-$k$, we adopt an annealing schedule for $k$ from $4 \to 1$, where each SAC training step begins with Top-$4$ for a specified number of gradient updates, followed by Top-$1$ with augmented samples, and concludes with Top-$1$ on actual samples. At the end of each training step, this yields a \dt Top-$1$ model corresponding to a hard DT, which is then used for interaction with the environment.

\dt demonstrates competitive performance with Deep RL on environments with limited state dimensions while consistently generating the most efficient DT policies. The performance advantage increases with environment complexity: tied for first on CartPole, leading by several percent on Acrobot, achieving substantial leads on discrete LunarLander ($\geq 28\%$) and Zerglings ($\geq 39\%$) compared to other DT architectures. While VIPER achieves comparable performance on LunarLander, it requires significantly larger trees ($>1300$ leaves vs $32$ for our method).

Continuous action spaces prove easier to handle than discrete cases, with all regressor-based architectures achieving higher rewards and showing more consistent performance across methods. This confirms that architecture choice is less critical when policies can use linear actuators rather than fixed discrete actions. The approximations used by DGT and ICCT, particularly the multiple uses of STEs, lead to reduced rewards. While \dt with a single STE shows some improvement, STE approximations still degrade performance. By avoiding these approximations, \dt with Top-$k$ maintains superior performance across environments, outperforming NN baselines by about 2\% on Continuous Lunar Lander and 3\% on Bipedal Walker, thereby demonstrating the advantage of approximation-free training of DTs.

\section{Conclusion}\label{sec:conslusion}

We proposed \dt, a differentiable architecture that is semantically equivalent to oblique DTs and can be trained end-to-end with standard gradient descent. For classification, this enables an exact, approximation-free procedure. For regression, we introduced an annealed Top-$k$ approach that overcomes the limitations of STE approximation, which suffers from biased gradients due to objective mismatch and instability from inconsistent parameter updates. By retaining forward and backward semantics during training, the annealed Top-$k$ approach ensures approximation-free learning. Together, these contributions make \dt the first architecture to achieve approximation-free training of oblique DTs in both classification and regression.
Our experiments show that \dt consistently produces more accurate trees than competing differentiable approaches. In classification, \dt reduces error by more than 10\% on harder benchmarks, establishing a clear margin over state-of-the-art methods. In regression, it achieves an improvement of over 3\% (on average) compared to baselines such as TAO-Linear, with substantially larger gains ($ > 10\%$) on challenging datasets. At the same time, compared to non-gradient-based methodologies (such as greedy, non-greedy, and global search approaches) for learning DTs, gradient-based methods are significantly faster. 
Finally, we demonstrated the application of \dt as a programmatic policy in RL, achieving performance comparable to or better than NN policies across both discrete and continuous action spaces. These results show that approximation-free differentiable trees combine the interpretability of symbolic models with the scalability of gradient-based training, opening promising directions for deploying interpretable yet high-performing models in safety-critical domains.

\newpage

\noindent {\bf Limitations:} \dt is a DT, making it unsuitable for high-dimensional inputs like images, where DTs struggle with complex shapes and require many leaves, negating their benefits. 

\medskip

\noindent {\bf Future Work:} For \dt, the choice of the height of the DT is treated as a hyperparameter, similar to the choice of the number of layers in an NN, in contrast to methods (e.g. CART \citep{breiman1984cart}) that grow trees height. For future work, we will consider developing differentiable methods for adaptive growth and pruning.


\acks{This research was conducted as part of the DesCartes program and was supported by the National Research Foundation, Prime Minister’s Office, Singapore, under the Campus for Research Excellence and Technological Enterprise (CREATE) program. This research/project is also supported by the National Research Foundation, Singapore and DSO National Laboratories under the AI Singapore Programme (AISG Award No: AISG2-RP-2020-017). 
The second author is partly supported by ANR-23-PEIA-0006 SAIF.
The computational work for this research was partially performed using resources provided by the NSCC, Singapore.}

\bibliography{25-2047}

\end{document}